\documentclass{article}

\usepackage[preprint]{neurips_2026}

\usepackage[utf8]{inputenc} 
\usepackage[T1]{fontenc}    
\usepackage{hyperref}       
\usepackage{url}            
\usepackage{booktabs}       
\usepackage{amsfonts}       
\usepackage{nicefrac}       
\usepackage{microtype}      
\usepackage{xcolor}         
\usepackage{bm}
\usepackage{amsmath}
\usepackage{amsthm}
\usepackage{pgfplots}
\usepackage{graphicx}
\usepackage{subcaption}
\usepackage{sansmath}
\usepackage{float}
\usepackage{makecell}
\pgfplotsset{compat=1.18}

\usetikzlibrary{arrows.meta,decorations.pathreplacing,positioning}

\definecolor{blockyellow}{RGB}{210,205,0}
\definecolor{barblue}{RGB}{58,105,135}
\definecolor{gridcolor2}{RGB}{0,0,0}
\definecolor{shade2}{RGB}{128,128,128}

\definecolor{bluegreen1}{RGB}{0,83,93}
\definecolor{beige1}{RGB}{238,221,192}

\newcommand{\fp}{f_\text{pred}}
\newcommand{\hz}{\hat{\mathbf{z}}}
\newcommand{\mygrid}[6]{
  \begin{scope}[shift={(#1,#2)}]
    \foreach \c/\r in {#5}{
      \fill[barblue] (\c-1,\r-1) rectangle ++(1,1);
    }

    \pgfmathtruncatemacro{\Cols}{#3}
    \pgfmathtruncatemacro{\Rows}{#4}
    \pgfmathtruncatemacro{\Colsm}{\Cols-1}
    \pgfmathtruncatemacro{\Rowsm}{\Rows-1}

    \draw[gridcolor2,line width=.9pt] (0,0) rectangle (\Cols,\Rows);

    \foreach \i in {1,...,\Colsm}{
      \draw[gridcolor2,line width=.75pt] (\i,0) -- (\i,\Rows);
    }
    \foreach \j in {1,...,\Rowsm}{
      \draw[gridcolor2,line width=.75pt] (0,\j) -- (\Cols,\j);
    }

    \node[font=\large] at (\Cols/2,\Rows+1.1) {#6};
  \end{scope}
}

\newtheorem{lemma}{Lemma}
\newtheorem{theorem}{Theorem}

\title{Winner-Take-All bottlenecks enforce disentangled symbolic representations in multi-task learning}

\author{%
  Julian Gutheil\\
  Institute of Machine Learning and Neural Computation \\
  Graz University of Technology \\
  Graz, Austria\\
  \texttt{julian.gutheil@tugraz.at} \\
  \And
  Simon Hitzginger \\
  Institute of Machine Learning and Neural Computation \\
  Graz University of Technology \\
  Graz, Austria\\
  \texttt{simon.hitzginger@tugraz.at} \\
  \AND
  Robert Legenstein\thanks{Corresponding author} \\
  Institute of Machine Learning and Neural Computation \\
  Graz University of Technology \\
  Graz, Austria\\
  \texttt{robert.legenstein@tugraz.at} \\
}

\begin{document}

\maketitle

\begin{abstract}
 Winner-take-all (WTA) networks constitute a central circuit motif in cortical networks of the brain. In addition, WTA-like activations are abundant in modern deep learning models in the form of the softmax activation for example in attention layers of transformers. While their role in the extraction of latent factors has been studied for relatively simple generative models, their role in the context of highly non-linearly entangled latent factors has remained elusive. In this article, we show that a WTA bottleneck within a deep neural network can enforce under certain well-defined conditions the extraction of categorical latent factors of the data in a multi-task learning setup. In particular, we prove that the representation that emerges in the WTA bottleneck is highly symbolic, where a single neuron or a population of neurons encodes the presence of a single abstract feature such as a specific object, color, or position. We furthermore show empirically on two datasets, that this also holds for architectures and setups that do not fully comply with the assumptions of our theorem and demonstrate the advantages of the acquired symbolic representation for generalization. Our proposed model provides insights into the generalization capabilities of deep neural networks with WTA-like components and may serve as an interface between symbolic and subsymbolic AI systems. 
\end{abstract}

\section{Introduction}
\label{sec:intro}

Winner-take-all (WTA) networks have been proposed to constitute a central circuit motif in cortical networks \citep{douglas1989canonical}. In biology, WTA motifs arise from lateral inhibition where excitatory neurons project to inhibitory neuron populations which in turn provide inhibitory feedback to these excitatory cells \citep{avermann2012microcircuits}. Theoretical studies have emphasized the role of WTA motifs in the context of learning, showing that local spike timing-dependent plasticity (STDP) in WTA architectures implement an expectation-maximization algorithm in simplified models that is able to extract the hidden causes of sensory data \citep{nessler2009stdp,legenstein2017probabilistic,jonke2017feedback}. Even earlier, WTA architectures were investigated in the context of competitive learning \citep{rumelhart1985feature}. Apart from their biological relevance, WTA-like activations are abundant in modern deep learning models in the form of the softmax activation, which can be viewed as a differential version of the hard WTA function. Besides its role as a normalizer when computing the probability distribution over discrete classes in the final classification layer of classical feed-forward and recurrent network architectures, it is also a central component of attention layers in transformers \citep{vaswani2017attention}.

While single-layer WTA networks have been shown to enable the extraction of latent factors from sensory input under the assumption of relatively simple generative models \citep{nessler2009stdp,legenstein2017probabilistic}, their role in the context of observations that arise from highly non-linearly entangled latent factors has remained elusive. In this article, we show that a WTA bottleneck within a deep neural network --- that is, a layer consisting of several WTA activations --- enforces under certain conditions the extraction of categorical latent factors of the data in a multi-task learning setup. 
In particular, we consider a deep neural network with a WTA bottleneck in its penultimate layer that observes an arbitrary information-preserving transformation of categorical latent factors $\mathbf{z}$ at its inputs. We prove that, when the network was trained to perfectly solve a large enough number of linear classification tasks on the latent factors $\mathbf{z}$, then the representation $\hat{\mathbf{z}}$ of the WTA bottleneck is a permutation of $\mathbf{z}$ that is further constrained by the structure of the WTA bottleneck (referred to as {\em structured permutation}).

We then show empirically on two data sets, that similar representations also emerge for architectures that do not fully comply with the assumptions of our theorem. 
The resulting representation  is disentangled  \citep{higgins2018towards} and highly symbolic, where a single neuron or a population of neurons encodes the presence of a single abstract feature such as a specific object, color, or position. Such representations are often called factorized representations \citep{johnston2023abstract,bengio2013representation}. We refer to this specific type of representation where latent features are encoded by populations of neurons as {\em symbolic representations}. It has been argued that disentangled representations benefit generalization \citep{bengio2013representation,johnston2023abstract}  (but see \citep{locatello2019challenging}).
We demonstrate this for the acquired symbolic representations by considering generalization to various types of test settings. 
Our approach to disentangled (symbolic) representation learning has the potential to link subsymbolic and symbolic AI and could therefore play a key role in neurosymbolic AI.

\section{Related work}

\paragraph{Learning in WTA architectures} First ideas regarding learning in competitive neural ensembles date back to \cite{rumelhart1985feature,rumelhart1986feature}. They considered a single-layer neural network where the neuron with the largest activation (the winner) updates its weights, leading to a clustering of the input space. Later, this idea was investigated in the context of theoretical neuroscience research. In \cite{nessler2009stdp}, it was shown that STDP enables WTA circuits to extract the latent features of its inputs in a well-defined graphical model. This work was generalized to networks with more biologically plausible soft inhibition \citep{jonke2017feedback,legenstein2017probabilistic} and to sheets of neurons with lateral inhibition implementing multiple WTA circuits \citep{bill2015distributed}. WTA architectures have also been studied in the context of hierarchical object recognition \citep{riesenhuber1999hierarchical} and assembly creation for long term memory \citep{legenstein2018long}. See \citep{maass2000computational} for an analysis of the computational power of WTA activations.

\paragraph{Disentangled representations through multi-task learning}
A large number of methods for disentangled representation learning (DRL) exist, see \citep{bengio2013representation,wang2024disentangled} for reviews. Recently, the role of multi task learning for DRL has been studied \citep{johnston2023abstract,vafidis2025disentangling}. In particular, \cite{vafidis2025disentangling} have proven that nonlinearly entangled continuous latent factors can be disentangled through multi task learning given a set of sufficiently diverse tasks that depend linearly on the latent factors. This result is remarkable as it provides a theoretical guarantee for the emergence of disentangled representations. We provide a related result for categorical latent variables with a surprising twist. We find that, when multi task learning is performed on a WTA bottleneck, then one can not only recover a linear combination of categorical latent variables, but the variables themselves.

\section{Problem formulation}
\label{sec:problem}

We now formally define the setup under which we study the emergence of symbolic representations.
\paragraph{Latent features:} We consider a set of $m$ latent categorical random variables $Z_1, \dots , Z_m$, where $Z_i$ can take values in $\{1, \dots, l_i\}$, with $l_i \in \mathbb{N}$, $l_i\ge 2$ as the number of categories. The corresponding one-hot vectors are $\mathbf{z}^{(1)}, \dots ,\mathbf{z}^{(m)}$, with $\mathbf{z}^{(i)} \in \{0,1\}^{l_i}$. The total latent vector $\mathbf{z}$ is then the concatenation of these vectors $\mathbf{z}=(\mathbf{z}^{(1)}; \dots ; \mathbf{z}^{(m)})^{\top} \in \{0,1\}^l$ with $l=\sum_i l_i$ (see matrix $C$ in Fig.~\ref{fig:setup}b). We denote by $\mathcal{Z}$ the set of all possible total latent vectors defined in that way.
\paragraph{Symbolic representations:} We first compare notions of disentangled representations and specifically define our notion of symbolic representations. \cite{vafidis2025disentangling} and \cite{ostojic2024computational} defined abstract and disentangled representations for continuous latents, which does not necessarily apply directly to categorical variables. Nevertheless, we discuss these notions here in order to bring our notion of symbolic representations into context. According to \citep{vafidis2025disentangling}, an abstract representation of the latent variables 
$z_1,\dots, z_m$ represents each $z_i$ linearly and approximately mutually orthogonally. Disentangled representations encode each $z_i$ orthogonally, without the necessity of linearity. Note that under this definition, axis-alignment is not required, i.e., it is not required that a given neuron represents a single latent factor. We demand axis-alignment for symbolic representations in the following sense. We say that a representation $\hz=(\hat z_1, \dots, \hat z_{l'})^\top$ is a symbolic representation of the total latent variable vector $\mathbf{z}=(z_1, \dots , z_l)^{\top}$ if for each $i$ there exists a set of neurons such that $z_i=1$ implies that at least one neuron in the set is active and furthermore for each neuron in the set, activity of the neuron implies $z_i=1$. In other words, the presence of the latent category can easily be checked by checking whether at least one neuron in the corresponding set is active. This includes the special case when each category is represented by exactly one neuron. A formal definition is given in Appendix \ref{app:Problem}.
\paragraph{Observables:} 
We observe a nonlinearly transformed version of $\mathbf{z}$, i.e., we a assume an injective function $\Phi: \{0,1\}^l \rightarrow \mathbb{R}^d$, see Fig.~\ref{fig:setup}a. For a given $\mathbf{z}$, the observable vector $\mathbf{x}$ is then
\begin{equation}
    \mathbf{x} = \Phi(\mathbf{z}).
\end{equation}
\paragraph{WTA Encoder:} We consider a winner-take-all (WTA) encoder $f_\text{enc}:\mathbb{R}^d \rightarrow \{0,1\}^l$ which takes as input $\mathbf{x}$  and outputs $\hat{\mathbf{z}}=(\hat{\mathbf{z}}^{(1)}; \dots ; \hat{\mathbf{z}}^{(m)})^{\top}$, see Fig.~\ref{fig:setup}a. $\hat{\mathbf{z}}$ consists of $m$ binary one-hot vectors which are concatenated together and $\hat{\mathbf{z}}$ has the same structure as the latents $\mathbf{z}$ (we assume identical structure in the proof, although the order of the $l_i$'s can be permuted w.l.o.g. Experiments will be performed with differing structures). The WTA encoder $f_\text{enc}$ consists of a feed-forward neural network $f_{\text{MLP}}: \mathbb{R}^d \rightarrow \mathbb{R}^l$ resulting in an output $\mathbf{a} \in \mathbb{R}^{l}$ and $m$ WTA heads $f_{\text{WTA}}$ applied on top of $\mathbf{a}$. The output $\mathbf{a}$ of $f_{\text{MLP}}$ is grouped into $m$ output groups $\mathbf{a}^{(i)} \in \mathbb{R}^{l_i}$. One WTA function $f_{\text{WTA}}: \mathbb{R}^{l_i} \rightarrow \{0,1\}^{l_i}$ is applied to each of these groups individually. The function $f_{\text{WTA}}$ sets its output corresponding to the maximum value of $\mathbf{a}^{(i)}$ to $1$ and all other values to $0$ (we assume that there is a unique largest component).
The output of the $i$-th WTA function is denoted by $\hat{\mathbf{z}}^{(i)}$ and the output vectors of all WTAs are concatenated in one vector $\hat{\mathbf{z}} \in \{0,1\}^l$
\begin{equation}
    \hat{\mathbf{z}}^{(i)} = f_{\text{WTA}}(\mathbf{a}^{(i)}), \quad \hat{\mathbf{z}}= (\hat{\mathbf{z}}^{(1)};...;\hat{\mathbf{z}}^{(m)})^{\top}.
\end{equation}
We refer to this group of WTA heads as a multi-WTA head.
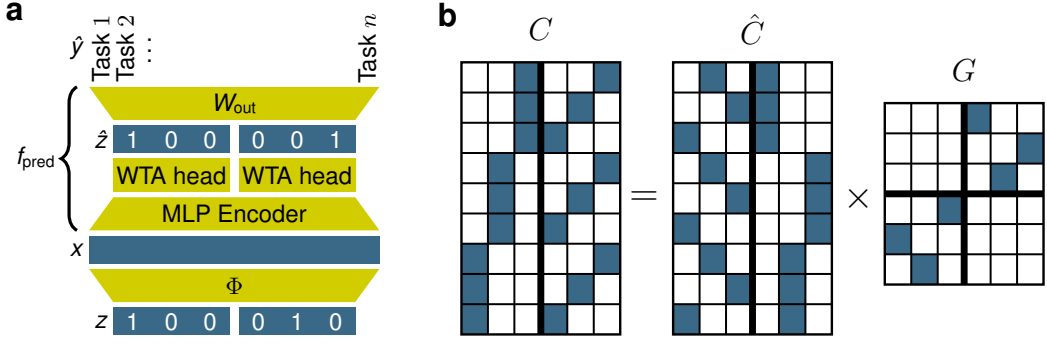
\begin{figure}
    \centering
    \begin{minipage}[t]{0.36\linewidth}
        \vspace{0pt}
        \centering
        \begin{subfigure}[t]{\linewidth}
            \vspace{2pt}
            \centering
            \begin{tikzpicture}[
                arr/.style={-{Latex[length=3mm,width=2mm]}, very thick},
                brace/.style={decorate, decoration={brace, amplitude=6pt}, very thick},
                every node/.style={inner sep=1pt},
                x=0.64cm,
                y=0.73cm,
            ]
    
            \node[rotate=90,font=\small \sffamily] at (-0.25,5.2) {Task $1$};
            \node[rotate=90,font=\small \sffamily] at (0.25,5.2) {Task $2$};
            \node[rotate=90,font=\small \sffamily] at (0.75,5.2) {\dots};
            \node[rotate=90,font=\small \sffamily] at (5.25,5.2) {Task $n$};
            
            \node[font=\sansmath \small] at (-0.75,5.2) {$\hat{\bm{y}}$};
            \node[font=\sansmath \small \sffamily] at (-1.6,3.2) {$\bm{f}_\mathrm{pred}$};
            \node[font=\sansmath \small] at (-0.75,1.55) {$\bm{x}$};
            \node[font=\sansmath \small] at (-0.25,0.25) {$\bm{z}$};
            
            \draw[brace] (-0.75,1.9) -- (-0.75,4.5);
            
            \fill[blockyellow]
                (0,3.9) -- (-0.5,4.5) --
                (5.5,4.5) -- (5,3.9) -- cycle;
            \node[font=\sansmath \small] at (2.5,4.2) {$\bm{W}_{\mathrm{out}}$};
            
            \node[font=\sansmath \small] at (-0.25,3.55) {$\hat{\bm{z}}$};
            
            \fill[barblue] (0,3.3) rectangle (2.4,3.8);
            \fill[barblue] (2.6,3.3) rectangle (5,3.8);
            
            \node[white,font=\sansmath \small \bfseries] at (0.40,3.55) {$1$};
            \node[white,font=\sansmath \small \bfseries] at (1.2,3.55) {$0$};
            \node[white,font=\sansmath \small \bfseries] at (2,3.55) {$0$};
            
            \node[white,font=\sansmath \small \bfseries] at (3,3.55) {$0$};
            \node[white,font=\sansmath \small \bfseries] at (3.8,3.55) {$0$};
            \node[white,font=\sansmath \small \bfseries] at (4.6,3.55) {$1$};
            
            \fill[blockyellow] (0,2.6) rectangle (2.4,3.2);
            \fill[blockyellow] (2.6,2.6) rectangle (5,3.2);
            
            \node[font=\small \sffamily] at (1.2,2.9) {WTA head};
            \node[font=\small \sffamily] at (3.8,2.9) {WTA head};
            
            \fill[blockyellow]
                (-0.5,1.9) -- (0,2.5) --
                (5,2.5) -- (5.5,1.9) -- cycle;
            \node[font=\sffamily \small] at (2.5,2.2) {MLP Encoder};
            
            \fill[barblue] (-0.5,1.3) rectangle (5.5,1.8);
            
            \fill[blockyellow]
                (-0.5,1.2) -- (5.5,1.2) --
                (5,0.6) -- (0,0.6) -- cycle;
            \node[font= \sffamily \small] at (2.5,0.9) {$\Phi$};
            
            \fill[barblue] (0,0) rectangle (2.4,0.5);
            \fill[barblue] (2.6,0) rectangle (5,0.5);
            
            \node[white,font=\sansmath \small \bfseries] at (0.40,0.25) {$1$};
            \node[white,font=\sansmath \small \bfseries] at (1.2,0.25) {$0$};
            \node[white,font=\sansmath \small \bfseries] at (2,0.25) {$0$};
            
            \node[white,font=\sansmath \small \bfseries] at (3,0.25) {$0$};
            \node[white,font=\sansmath \small \bfseries] at (3.8,0.25) {$1$};
            \node[white,font=\sansmath \small \bfseries] at (4.6,0.25) {$0$};
            \node[anchor=north west, font=\sffamily\bfseries\large,
            xshift=-4pt,yshift=4pt] at (current bounding box.north west) {a};
            \end{tikzpicture}
        \end{subfigure}
    \end{minipage}
    \hspace{0.35cm}
    \begin{minipage}[t]{0.6\linewidth}
        \vspace{0pt}
        \centering
        \begin{subfigure}[t]{\linewidth}
            \centering
            \begin{tikzpicture}[x=0.35cm,y=0.4cm]
            
            \mygrid{0}{0}{6}{9}
            {
              3/9,6/9,
              3/8,5/8,
              3/7,4/7,
              2/6,6/6,
              2/5,5/5,
              2/4,4/4,
              1/3,6/3,
              1/2,5/2,
              1/1,4/1
            }
            {$C$}
            \draw[line width=2.6pt] (3,0) -- (3,9);
            
            \node[font=\Large] at (6.9,4.5) {$=$};
            
            \mygrid{8}{0}{6}{9}
            {
              2/9,4/9,
              3/8,4/8,
              1/7,4/7,
              2/6,6/6,
              3/5,6/5,
              1/4,6/4,
              2/3,5/3,
              3/2,5/2,
              1/1,5/1
            }
            {$\hat{C}$}
            \draw[line width=2.6pt] (11,0) -- (11,9);
            
            \node[font=\Large] at (15.0,4.5) {$\times$};
            
            \mygrid{16}{1.65}{6}{6}
            {
              4/6,
              6/5,
              5/4,
              1/2,
              2/1,
              3/3
            }
            {$G$}
            \draw[line width=2.6pt] (19,1.65) -- (19,7.65);
            \draw[line width=2.6pt] (16,4.65) -- (22,4.65);
            \node[anchor=north west, 
            font=\sffamily\bfseries\large,
            yshift=4pt,
            xshift=-12pt] at (current bounding box.north west) {b};
            \end{tikzpicture}
        \end{subfigure}
    \end{minipage}
    \caption{{\textbf{Theoretical framework.} \textbf{a)} Network setup, showing representations (blue) and mappings (olive). Latent vector $\mathbf{z}$ (bottom) examplified for two latent variables with three categories each. Each latent is coded as a one-hot vector. The latents are mapped through an injective function $\Phi$ to an entangled represenation $\mathbf{x}$. The WTA encoder $f_\textnormal{enc}$ maps $\mathbf{x}$ to a representation $\hz$, which is constrained by a multi-WTA head. A readout layer (top) is applied to $\hz$ and trained to solve $n$ tasks. \textbf{b)} Structured permutation. Matrix $C$ illustrates the "code" matrix for two latent variables with three categories each (blue indicates value $1$, white value $0$). The rows together encode all possible latent vectors $\mathbf{z}$ in lexicographical order. The columns of matrix $\hat C$, are a structured permutation of the columns of $C$ in the sense that the latent variables are permuted and categories within the latent variables are permuted. $C$ can be obtained from $\hat{C}$ by $C = \hat{C} G$ for structured permutation matrix $G$. } 
    }
    \label{fig:setup}
\end{figure}
\paragraph{Readout layer:} The representation $\hat{\mathbf{z}}$ serves as input for a fully connected readout layer $f_\text{readout}:\{0,1\}^l \rightarrow (0,1)^n$ with $n$ sigmoidal outputs, which produce the output $\hat{\mathbf{y}} \in \mathbb{R}^{n}$ according to
\begin{align}
    \hat{\mathbf{z}}_ {\text{out}} &= W_{\text{out}}\hat{\mathbf{z}} \label{eq:zout}\\
    \hat{\mathbf{y}} &=  \sigma(\hat{\mathbf{z}}_ {\text{out}}),
    \label{eq:taskout}
\end{align}
with weight matrix $W_{\text{out}} \in \mathbb{R}^{n \times l}$, see Fig.~\ref{fig:setup}a. Note that we do not use a bias vector in this layer. It could be incorporated, but as we will see below, it is not needed. 
$\sigma$ denotes the component-wise logistic sigmoid function. This layer will predict the posterior probabilities for $n$ linear tasks, see next paragraph. We refer to the encoder together with the readout layer as the predictor network $f_\text{pred}: \mathbb{R}^{d} \rightarrow (0,1)^n$, $f_\text{pred}:=f_\text{readout} \circ f_\text{enc}$.
\paragraph{Definition of linear tasks:} We define $n$ linear binary classification tasks based on the latent categorical variable vector $\mathbf{z}$. 
The $i$-th task assumes that $\mathbf{z}$ is originating from one of two classes indicated by $A_i=1$ (class 1) and $A_i=0$ (class 0) with respective prior probabilities $P(A_i=1)$ and $P(A_i=0)$. 
If $\mathbf{z}$ originated for class 1 (resp.~0), it is assumed that each $\mathbf{z}^{(j)}$ was drawn from a categorical distribution with parameters $\mathbf{p}^{(i,j)}$ (resp.~$\mathbf{q}^{(i,j)}$). Again, we write $\mathbf{p}^{(i)}$ (resp.~$\mathbf{q}^{(i)}$) for the concatenation of these parameter vectors, i.e., $\mathbf{p}^{(i)}=(\mathbf{p}^{(i,1)};\dots;\mathbf{p}^{(i,m)})^{\top}$. 
Hence, the tuple $(\mathbf{p}^{(i)},\mathbf{q}^{(i)},P(A_i=1))$ defines task $i$. The task of the readout layer is to predict the posterior probability $P(A_i=1|\mathbf{z})$ that some $\mathbf{z}$ belongs to class $A_i=1$.
It is convenient to stack the posteriors of all tasks together. We thus define $P(A|\mathbf{z}) \in [0,1]^n$ as a vector with the $i$-th component given by $P(A_i=1|\mathbf{z})$. A standard derivation shows that this posterior can be written as  $P(A|\mathbf{z}) = \sigma \left(W \mathbf{z} + \mathbf{b}\right)$,
with $W=[w_{ij}]_{\substack{i=1,\dots,n\\j=1,\dots ,l}}$ and $\mathbf{b}=(b_1, \dots, b_n)^\top$ for $w_{ij}= \ln \frac{p^{(i)}_j}{q^{(i)}_j}$ and $b_i=\ln \frac{P(A_i=1)}{P(A_i=0)}$, see Appendix \ref{app:Problem}.
Note however that, since $\mathbf{z}$ consists of a constant number of 1-entries, $\mathbf{b}$ can be incorporated into $W$, so
\begin{equation}
    \label{eq:posteriorAllTask}
    P(A|\mathbf{z}) = \sigma \left(W \mathbf{z}\right),
\end{equation}
for a suitable re-definition of W (see Appendix \ref{app:Problem}).

This shows that we can write $P(A|\mathbf{z})$ as a linear mapping of the latent vector $\mathbf{z}$ transformed by a sigmoid function. 

\paragraph{Structured permutation:} We will show that under multi-task learning, the output of the multi-WTA head evolves a symbolic representation of the latent variables. In particular, in the theoretical setup, the latent variables are recovered exactly. Obviously, this reconstruction can only be done up to some permutation, as the order of latents and the order of categories within the latents is arbitrary. That means, that in the representation $\hat{\mathbf{z}}$, the latent variables can be permuted and within the latent variables, the categories can be permuted relative to $\mathbf{z}$. Also, a variable $Z_i$ can only be represented by WTA head $j$ if its number of categories matches the number of outputs of the WTA gate. It is therefore a structured permutation that depends on the structure of the latents. We illustrate a structured permutation in Fig.~\ref{fig:setup}b and formally define it in Appendix \ref{app:Problem}. Note that a representation that is a structured permutation of the latents is by definition a symbolic representation. We will show in simulations that, when we relax some conditions of the theoretical setup (where a structured permutation is not defined anymore), still, symbolic representations emerge.

\section{Theoretical Results}
\label{sec:theoresults}

The predictor network $f_\text{pred}$ is trained to predict the posterior probabilities $P(A|\mathbf{z})$. Note that due to our assumption that the nonlinear entanglement transformation $\Phi$ is injective and $\mathbf{x}=\Phi(\mathbf{z})$, we have $P(A|\mathbf{z})=P(A|\mathbf{x})$. In other words, all information is available to predict this posteriors from the entangled representation in $\mathbf{x}$. Assume that the predictor network has sufficient capacity and has been perfectly trained, i.e., it is an optimal predictor. Formally, we say that a predictor $\fp$ is optimal (with respect to posterior distributions $P(A|\mathbf{z})$ and entanglement transformation $\Phi)$ if 
\begin{equation}
\label{eq:opt_predictor}
\fp(\Phi(\mathbf{z}))=P(A|\mathbf{z}) \qquad \text{for all } \mathbf{z} \in \mathcal{Z}.   
\end{equation}
We first show that in this case, $\mathbf{z}$ is linearly decodable from the readout activations (Eq.~\ref{eq:zout}) $\hat{\mathbf{z}}_{\text{out}}$ as well as from the multi-WTA output $\hat{\mathbf{z}}$ if $W$ has full rank (recall that $W$ defines the tasks, see Eq.~\eqref{eq:posteriorAllTask}).

As defined above, $\hat{\mathbf{y}}=\fp(\Phi(\mathbf{z}))$ and by assumption, $\fp$ is an optimal predictor. We combine Eqs.~\eqref{eq:posteriorAllTask} and \eqref{eq:opt_predictor}
\begin{equation}
    \hat{\mathbf{y}} = \sigma \left(W \mathbf{z}\right).
    \label{eq:yz}
\end{equation}

We see that the latent vector $\mathbf{z}$ is encoded in the posterior prediction $\hat{\mathbf{y}}$. We know from Eq.~\eqref{eq:taskout} that $\hat{\mathbf{y}} =  \sigma(\hat{\mathbf{z}}_ {\text{out}})$. Therefore, we combine Eq.~\eqref{eq:yz} and Eq.~\eqref{eq:taskout} and get
\begin{equation}
    \sigma(\hat{\mathbf{z}}_ {\text{out}}) = \sigma \left(W \mathbf{z} \right)
    \quad \Leftrightarrow  \quad \hat{\mathbf{z}}_ {\text{out}} =  W \mathbf{z}.
    \label{eq:zoutlinear}
\end{equation}

To solve Eq.~\eqref{eq:zoutlinear} for $\mathbf{z}$, we can use the pseudo-inverse and get
\begin{equation}
    \mathbf{z}=(W^{\top}W)^{-1}W^{\top} \hat{\mathbf{z}}_ {\text{out}}.
\end{equation}
Hence, $\mathbf{z}$ is linearly decodable from $\hat{\mathbf{z}}_ {\text{out}}$ if $W^{\top}W$ is invertible, that is, if $W$ has full rank. Since $\hat{\mathbf{z}}_ {\text{out}}=W_{\text{out}} \hat{\mathbf{z}}$, this implies
\begin{equation}
    \mathbf{z}=(W^{\top}W)^{-1}W^{\top}W_{\text{out}}\hat{\mathbf{z}}.
    \label{eq:zeqz}
\end{equation}

Hence, $\mathbf{z}$ can be obtained by a linear mapping from $\hat{\mathbf{z}}$ if $W$ has full rank, which holds for each $\mathbf{z} \in \mathcal{Z}$ and its corresponding encoding $\hat{\mathbf{z}}$. We can arrange all $\mathbf{z} \in \mathcal{Z}$ in a code matrix $C$ where each row is one latent vector $\mathbf{z}$. This can be done using an arbitrary order of the the rows, but for simplicity we can imagine it in lexicographical order as shown for an example in Fig.~\ref{fig:setup}b. Each vector has length $l=\sum_i l_i$ and there are $p := \prod_i l_i$ such vectors. Hence we have $C\in \{0,1\}^{p \times l}$. We define a similar code matrix $\hat C$ for the encodings $\hat{\mathbf{z}}$. From Eq.~\eqref{eq:zeqz}, we know that we can write 
\begin{equation}
    \label{eq:linmap}
    C=\hat C G
\end{equation}
for $G=((W^{\top}W)^{-1}W^{\top}W_{\text{out}})^\top$. Our main theoretical result is that the multi-WTA representation $\hat{\mathbf{z}}$ obtained in this way is a structured permutation of the latent variables $\mathbf{z}$, i.e., the columns of $\hat C$ are a structured permutation of the columns of $C$. In other words, due to the multi-WTA bottleneck, one obtains a more strongly constrained symbolic representation. We prove this for the case where the number of latent categories is the same length $l_\text{c}$ for all latents, i.e., $l_i=l_\text{c}$ for all $i$, for simplicity, but a generalization to mixed latent category sizes seems straight-forward. We say that a matrix $C$ has an $m$-fold one-hot column-block structure if every row of $C$ consists of $m$ consecutive blocks where each block has exactly one $1$ entry and the other entries in the block are $0$. Both matrices $C$ and $\hat C$ have an $m$-fold one-hot column-block structure by design, where each block is of length $l_\textnormal{c}$. The proof of the following theorem is provided in Appendix \ref{app:Proof_permutation}.
\begin{theorem}
\label{thm:Theorem_main}
    Let $m, l_\textnormal{c}$ be natural numbers with $l_\textnormal{c}\ge 2$. Let $l=m l_\textnormal{c}$ and $p=l_\textnormal{c}^m$. Let $C,\hat C \in \{0,1\}^{p \times l}$ be two matrices with an $m$-fold one-hot column-block structure. Both $C$ and $\hat C$ consist of all possible rows of that type. 
Assume that $C= \hat C G$ for some matrix $G\in \mathbb{R}^{l \times l}$. Then, the columns of $\hat C$ are a structured permutation of the columns of $C$.
\end{theorem}

\section{Experiments}
\label{sec:results}
Our analysis in Section \ref{sec:theoresults} revealed that categorical latent variables can be recovered from the multi-WTA bottleneck of an optimal multi-task predictor. This raises the question whether such a predictor can be trained. After all, WTA nonlinearities are not differentiable and potentially hard to train as a representation has to be found by the optimization algorithm that maps categories of each latent variable to a single WTA gate. Intuitively, this may induce a large number of local optima in the error landscape. A potential solution could be to replace hard WTA gates by softmax units. However, our simulations showed that this approach tends to produce blury output distributions at the WTA gates (i.e., there is no clear winner), which turned out to be insufficient for the extraction of a symbolic representation of latent variables. 
We therefore adopted an approach with a hard multi-WTA bottleneck. To resolve the gradient problem and to encouraging exploration, we used the Straight-Through Gumbel-Softmax Estimator with Gumbel noise \citep{jang2017categorical} in combination with a high learning rate and a high temperature coefficient.

\subsection{Symbolic representations from nonlinearly entangled latent variables}
\label{sec:results_symbolic}

In order to obtain data with known ground-truth latents, we generated inputs $\mathbf{x}$ from given latent vectors $\mathbf{z}$ via a nonlinear transformation $\Phi$. We tested our theory in a first setup with $5$ latent variables $Z_1,\dots,Z_5$, each with $5$ categories. The latent variables were sampled from independent categorical distributions with uniform probabilities. The transformation $\Phi$ was obtained by a Multi-Layer Perceptron (MLP) with four hidden layers and Leaky ReLU nonlinearities (see Appendix \ref{app:details_symbolic}). The resulting representation vectors $\mathbf{x} \in \mathbb{R}^{70}$ established an entangled representation (see Section \ref{sec:results_OOD} for a discussion). The multi-WTA encoder was implemented as an MLP with four hidden layers and $5$ WTA heads. We used 5 outputs for each WTA head to get a WTA representation $\hat{\mathbf{z}}$ with the same structure as the latent variables, consistent with our theory. Since the WTA structure matches the latent structure as assumed in the theory, we refer to this setup as the {\em matched} setup.

A single fully connected layer was applied to $\hz$ with $n$ independent logistic sigmoids as output nonlinearities, resulting in the task prediction vector $\hat{\mathbf{y}}\in (0,1)^n$ ($n$ is the number of tasks). 
We generated a set of $n=30$ tasks by randomly sampling a tuple $(\mathbf{p}^{(i)},\mathbf{q}^{(i)},P(A_i=1))$ for each task $i$ (see Section \ref{sec:problem}). The resulting posterior class probabilities were used as targets to minimize the binary cross entropy of the network output. 
We trained the network on $100,000$ input samples for $1,000$ epochs and tested the network on $10,000$ samples. We performed this training procedure five times with different random seeds. 
The tasks were solved perfectly in $3$ out of the $5$ runs (Mean Absolute Error (MAE) less than $10^{-6}$ on the test set). Consistent with our theory, those models which solved the tasks evolved a symbolic representation of $\mathbf{z}$ in $\hat{\mathbf{z}}$. In other words, each WTA head represented one latent variable with categories represented in permuted order with respect to $\mathbf{z}$. The mapping is illustrated in Appendix \ref{app:results_symbolic}, Fig.~\ref{fig:5x5mapping}.

We next tested our model in a setup that does not match the theoretical assumptions in Section \ref{sec:theoresults}. In practice, the number of categories per hidden variable is typically not known in advance. We therefore did not match the dimensionality of the WTA heads to the latent categories. Instead, we used $10$ outputs in each of the five WTA heads for latent variables $Z_1,\dots,Z_5$ with $l_1=5$, $l_2=8$, $l_3=5$, $l_4=3$, and $l_5=9$ categories respectively. In this setup we had more categories and therefore we used a larger entangled representational space $\mathbf{x} \in \mathbb{R}^{100}$.
We trained the model on $n=50$ tasks five times with different random seeds.
$4$ out of $5$ models solved the tasks perfectly.
The resulting representation $\hat{\mathbf{z}}$ of one model, which solved the tasks perfectly, is shown in Fig.~\ref{fig:repr}a. Again, each WTA head can be mapped to exactly one latent variable.
For a given WTA head, several outputs are mapped to one category such that they encode a category in a shared manner. Each category can be decoded by a logic OR over these outputs, hence the network has developed a symbolic representation of the latent variables. 
We evaluated the representations of each model (see Appendix \ref{app:details_symbolic}) and found that all four models that solved the tasks perfectly evolved a fully symbolic representation of the latents. The model which did not solve the tasks perfectly still recovered  $28$ of the $30$ categories correctly and mapped $3$ WTA heads correctly to $3$ latent factors (see Appendix \ref{app:results_symbolic}, Fig. \ref{fig:heatmap_fail_unmatched}). In this case, two neurons within a single WTA head $\hat{\mathbf{z}}^{(3)}$ jointly represented two categories of the second latent factor $\mathbf{z}^{(2)}$. If one of these two neurons was active, the value of $\mathbf{z}^{(2)}$ could not be decoded unambiguously. This ambiguity could be resolved with the help of WTA head $\hat{\mathbf{z}}^{(2)}$. This head had 10 neurons, and encoded the 3 categories of $\mathbf{z}^{(4)}$. So it was able to encode additional information about $\mathbf{z}^{(2)}$ by using different neurons for a single category of $\mathbf{z}^{(4)}$.

In summary, our results show that multi-task learning with WTA bottlenecks can extract symbolic representations also in the case when the multi-WTA bottleneck exceeds the number of latent categories. However, as the bottleneck size increases, the network representation becomes less constrained and training can find loopholes to represent latents in a partly entangled manner.

\input{figureMappingOODonly}

\subsection{Symbolic representations improve generalization}
\label{sec:results_OOD}

We next compared the generalization capability of classifiers trained on the evolved symbolic representation $\hat{\mathbf{z}}$ to that of classifiers trained on the entangled representation $\mathbf{x}$ in the unmatched setup described above. We examined three different types of generalization.  

First, we randomly split the set $\mathcal{Z}$ of all possible $5400$ latent vectors into three disjoint sets, a training set $\mathcal{Z}_{\text{train}}$, a validation set $\mathcal{Z}_{\text{val}}$, and a test set $\mathcal{Z}_{\text{test}}$ (we refer to this split as the {\em random train-test split}). Consequently, the models were tested on category combinations that were not part of the training procedure.
We defined a new linear classification task on the latent variables and trained a multi layer perceptron $\mathrm{MLP}_{\mathrm{\mathbf{x}}}$ on the entangled representation $\mathbf{x}$ and another multi layer perceptron $\mathrm{MLP}_{\mathrm{\hz}}$ on the multi-WTA representation $\hz$. 
To ensure a fair comparison,
we used for both models the validation set for hyperparameter optimization (HPO) and tested them on the test set. We used the Area Under the Receiver Operating Characteristic curve (AUC)\citep{FAWCETTAUROC} to evaluate the generalization ability of both models in dependence of the training set size, see Fig.~\ref{fig:repr}b. Even for a training set size of $100$ samples (less than $2$\% of all possible latent vectors), $\mathrm{MLP}_{\mathrm{\hat{\mathbf{z}}}}$ exhibited nearly perfect generalization with a mean AUC of $0.993 \pm 0.005$ (mean $\pm$ standard deviation (SD)). The test performance of $\mathrm{MLP}_{\mathrm{\mathbf{x}}}$ was clearly worse even for $4400$ training samples ($\mathrm{AUC}=0.973 \pm 0.01$; train: $\mathrm{AUC}=0.991 \pm 0.007$). 

We next tested generalization to samples that include a combination of  feature-values of the first two latent variables that has never been observed during training (see Appendix \ref{app:detials_OOD}, we refer to this split as the {\em pair-of-categories train-test split}). The gap increased for this stronger type of generalization between both models, see Fig.~\ref{fig:repr}c. Again, for a training set size of $100$ samples, $\mathrm{MLP}_{\mathrm{\hat{\mathbf{z}}}}$ exhibited good generalization (test: $\mathrm{AUC}=0.987 \pm 0.006$; train: $\mathrm{AUC}>0.999$ ($\mathrm{SD}<0.001$)). The test performance of $\mathrm{MLP}_{\mathrm{\mathbf{x}}}$ at $4400$ training samples dropped strongly ($\mathrm{AUC}=0.917 \pm 0.059$; train: $\mathrm{AUC}=0.917 \pm 0.074$).

It was argued that disentangled representations possess the ability to automatically generalize over irrelevant factors. For example, if a learner has learned to distinguish red and green berries for one berry shape, then it can generalize this distinction to other berry shapes, as color is disentangled from shape \citep{johnston2023abstract,vafidis2025disentangling}. In fact, this test has been proposed to identify entangled representations \citep{johnston2023abstract}. 
To this end, we trained MLPs  on a task which did not depend on the first latent variable. Here, the training set contained only samples where the first latent variable has the constant value $1$. The test set consisted of all possible samples which were not used in the training procedure (see Appendix \ref{app:detials_OOD}, we refer to this split as the {\em constant-category train-test split}). $\mathrm{MLP}_{\mathbf{x}}$ achieved an AUC of $0.936 \pm 0.028$ on the training set and an AUC of $0.864 \pm 0.023$ on the test set. Hence, it did not generalize well, indicating an entangled representation. In contrast, $\mathrm{MLP}_{\hz}$ exhibited excellent generalization with an AUC of $>0.999$ ($\mathrm{SD}<0.001$) on the training set and an AUC of $0.999$ ($\mathrm{SD}<0.001$) on the test set.

\subsection{Training with confounding latent factors}
\label{sec:results_confound}

In the setups above, all latent factors were relevant and we assumed an injective map $\Phi$ from the latents to entangled representation $\mathbf{x}$. In reality, this mapping may be confounded by other factors that are neither observable nor relevant for the learner (and therefore considered noise). We emulated this situation by introducing additional confounding latent variables in $\mathbf{z}$ that did not influence the posterior of the tasks used to extract the WTA representation. From the viewpoint of the network, these confounding variables thus introduced noise in $\mathbf{x}$. 
To test the resilience of our model to such noise, we considered the unmatched setup as described above, with the addition of $5$ confounding latent variables with same structure as the original latent variables. We used a larger entangled representational space with $\mathbf{x} \in \mathbb{R}^{500}$, to avoid that the additional noise reduces the information of $\mathbf{z}$ in $\mathbf{x}$.  
We trained again five models with different random seeds, and got five models which were able to solve the tasks perfectly. Furthermore, all models developed a symbolic representation of the non-confounding latents. In one model, one WTA head did not exclusively encode only one latet factor but was correlated with two factors (see Appendix \ref{app:results_confound}). Nevertheless, the multi-WTA representation of this model was a symbolic representation according to our definition. 
The generalization performance of networks trained on these representations was comparable to the setup without confounding factors with a somewhat smaller gap between the two representations (see Appendix \ref{app:results_confound}).

\subsection{Number of tasks during multi-task learning}
\label{sec:results_number_of_tasks}
We next examined the number of tasks necessary during the multi-task learning to obtain a symbolic representation of latents. For this experiment we used the matched setup (which has $l=25$ latent categories) in order to stay as close as possible to the theory. We trained the model on $1$, $5$, $10$, $15$, $20$, $25$, and $30$ tasks. 
Consistent with the theory, with $25$ tasks, a symbolic representation emerged if the model was able to solve all tasks perfectly ($2$ out of $5$ models). However, this was also the case for $20$ tasks ($3$ out of $5$ models solved the tasks perfectly), indicating that fewer tasks suffice in some setups. When less than $20$ tasks were used, training did not converge to optimal performance (see Appendix \ref{app:results_num_tasks} for detailed results).

\subsection{Symbolic representations from visual input}
\label{sec:results_dsprites}

\input{figureDsprites}

So far, we defined the entanglement transformation $\Phi$ explicitly through a neural network. We next asked whether our results also apply to a categorical variant of the standard dsprites data set \citep{dsprites17} that has been used frequently to study disentanglement \citep{burgess2018understanding}. The dsprites data set consists of images of simple shapes (rectangle, circle, heart). We used non-rotated shapes at $64$ different locations and colorized them with $10$ different colors, resulting in four latent factors with $l_1=3$ (shape), $l_2=8$ (position x), $l_3=8$ (position y), $l_4=10$ (shape color). Hence, the transformation $\Phi$ is implicit in the generation of the input images.

We trained a model with four WTA heads, each comprising 10 output neurons, following the procedure described in Section~\ref{sec:results_symbolic}. The model was trained on $n = 50$ tasks for $1{,}000$ epochs, repeated across five random seeds. 
In these simulations, we added an $\ell_1$ regularization term with coefficient $3 \times 10^{-4}$ to the readout layer weights. This regularization promotes sparse connectivity, thereby discouraging the model from distributing information across multiple neurons within a WTA head and instead encouraging the use of minimal, selective subsets of neurons. Such sparsity facilitates the emergence of symbolic representations. Preliminary experiments without this regularization led to weaker symbolic representations. See Appendix \ref{app:details_dsprites} for details.

In three out of the five runs, the model converged to a perfect symbolic representation. See Fig.~\ref{fig:dsprites}a for an example.  In the remaining two runs, a symbolic representation emerged for $26$ and $28$ categories out of the total $29$ categories respectively.
Detailed results can be found in Appendix~\ref{app:results_dsprites}.
We also repeated the generalization experiment described in Section \ref{sec:results_OOD} for the dsprites data set, with a test set that includes pairs of latent variables that have never been observed during training (pair-of-categories train-test split). The results, summarized in Fig.~\ref{fig:dsprites}b and Table~\ref{tab:auc_dsprites}, again demonstrate a near-perfect generalization when using 1000 training samples ($\mathrm{AUC}>0.999$ ($\mathrm{SD}<0.001$)). In contrast, the performance of an MLP trained on the entangled representation ${\mathbf{x}}$ was substantially lower ($\mathrm{AUC}=0.938 \pm0.044$).

\section{Discussion}

We have proven that under certain conditions, neural networks with multi-WTA bottlenecks evolve representations that recover categorical latent factors of their inputs up to a structured permutation.
This result is surprising as it shows that categorical latent variables can be recovered exactly from arbitrarily entangled representations. The disentanglement does however not come for free. It necessitates a set of sufficiently different tasks that indirectly provide information about the latents. The number of tasks needed is however moderate. According to our theory, the full rank of the task matrix $W$ is sufficient, which implies that the number of tasks should be at least $l$, that is, the cumulative number of categories of the latent variables. Our simulation results indicate that in practical settings, symbolic representations emerge even for smaller numbers of tasks (Section \ref{sec:results_number_of_tasks}). Our theoretical result is based on the assumption that these tasks are predictions of posterior probabilities over classes (leading to a quasi-linear model). This raises the question whether other types of tasks that arise naturally in biological or practically relevant situations would induce the same or at least similar types of representations.  Natural candidates are self-supervised tasks such as autoencoding \citep{bank2023autoencoders,kingma2013auto} or prediction \citep{nagai2019predictive,oord2018representation}. Since our multi-WTA bottleneck naturally implements a sparse code, autoencoding relates our framework to sparse autoencoders \citep{makhzani2013k,makhzani2015winner}.

To derive our theoretical result, we assumed that the multi-WTA bottleneck has the same structure as the latent variables (i.e., the number of outputs of the WTAs match the number of number of categories of the latents). Our simulation results revealed that symbolic representations also emerge if the number of WTA outputs exceed the category number. This however can only be true up to some extent, as for large WTAs, the network can find easy ways to encode the values of latents in an entangled manner. It would be interesting to investigate whether there are conditions for the emergence of symbolic representations in this more general case. 

From the neuroscience perspective, our model fits well to the known importance of WTA motifs in cortical circuits \citep{douglas1989canonical,jonke2017feedback} and the sparse coding theory \citep{olshausen2004sparse}. It is also in line with symbol-like object representations found in higher cortical areas \citep{quiroga2005invariant}. Some authors have argued that linear mixtures of latent variables have advantages over (symbolic) axis-aligned representations in terms of a higher flexibility to learn new functions \citep{rigotti2013importance,johnston2023abstract,vafidis2025disentangling}. We argue that both types of representations have advantages (see also below), which means that likely a larger family of representations is extracted and utilized in cortex.

Both, in terms of machine learning and neuroscience, symbolic representations have the advantage that they can function as an interface to symbolic reasoning modules, giving rise to high-level cognitive capabilities such as abstraction, variable binding \citep{feldman2013neural,muller2020model}, and mathematical skills. Therefore, symbolic feature extractors could play a key role in the rapidly growing field of neurosymbolic artificial intelligence. Our theory also hints towards an explanation for the miraculous generalization capabilities of architectures based on softmax attention \citep{vaswani2017attention}.

\begin{ack}
This research was funded in whole or in part by the Austrian Science Fund (FWF) [10.55776/COE12] (JG, SH, RL), and by NSF EFRI grant \#2318152 (RL).
\end{ack}

{
\bibliographystyle{unsrt}
\bibliography{literature}
}

\clearpage

\appendix

\setcounter{figure}{0}
\setcounter{table}{0}

\renewcommand{\thefigure}{A\arabic{figure}}
\renewcommand{\thetable}{A\arabic{table}}

\section{Technical appendices and supplementary material}

\subsection{Definitions and derivations for Section \ref{sec:problem}}
\label{app:Problem}

\paragraph{Definition of symbolic representations:}
Consider a representation $\hz=(\hat z_1, \dots, \hat z_{l'})^\top \in \{0,1\}^{l'}$ and a total latent vector $\mathbf{z}=(z_1, \dots, z_l)^\top$.

We say that a category $z_i$ is represented in a symbolic manner if there exists a subset $\mathcal I^{(i)} \subseteq \{1, \dots ,l'\}$ such that
\begin{enumerate}
    \item $z_i = 0 \Rightarrow \hat z_j = 0 \quad \forall j \in \mathcal I^{(i)}$,
and
\item $z_i = 1 \Rightarrow \exists j \in \mathcal I^{(i)}\text{ such that } \hat z_j = 1$.
\end{enumerate}
We refer to the number of categories for which these conditions hold as \textit{number of symbolic-encoded categories}.
We say that a representation $\hz$ is a symbolic representation of $\mathbf{z}$, if all categories $z_i$ of $\mathbf{z}$ are represented in a symbolic manner in $\hz$.

\paragraph{Definition of linear tasks:} We define $n$ linear binary classification tasks based on the latent categorical variable vector $\mathbf{z}$. The $i$-th task assumes that $\mathbf{z}$ is originating from one of two classes indicated by $A_i=1$ (class 1) and $A_i=0$ (class 0) with respective prior probabilities $P(A_i=1)$ and $P(A_i=0)$. If $\mathbf{z}$ originated from class 1 (resp.~0), it is assumed that each $\mathbf{z}^{(j)}$ was drawn from a categorical distribution with parameters $\mathbf{p}^{(i,j)}$ (resp.~$\mathbf{q}^{(i,j)}$). Again, we write $\mathbf{p}^{(i)}$ (resp.~$\mathbf{q}^{i}$) for the concatenation of these parameter vectors, i.e., $\mathbf{p}^{(i)}=(\mathbf{p}^{(i,1)};...;\mathbf{p}^{(i,m)})^{\top}$. 

A standard and well-known derivation shows that the posterior $P(A_i=1|\mathbf{z})$, that $\mathbf{z}$ belongs to class $A_i=1$, can be written as
\begin{align}
    P(A_i=1|\mathbf{z}) &= \sigma \left(\ln \frac{P(A_i=1, \mathbf{z})}{P(A_i=0,\mathbf{z})}\right)\\
    &= \sigma  \left(\ln \frac{P(\mathbf{z}|A_i=1)}{P(\mathbf{z}|A_i=0)} + \ln \frac{P(A_i=1)}{P(A_i=0)}\right).
    \label{eq:posterior}
\end{align}

We can write $P(\mathbf{z}^{(j)}|A_i=1)=\prod_{k=1}^{l_j}  (p^{(i,j)}_k)^{z^{(j)}_{k}}$, with $k$ denoting the category index of a latent categorical variable $j$. We assume that the latent categorical variables are independent. Therefore, we can multiply over all categoricals and write $P(\mathbf{z}|A_i=1)=\prod_{j=1}^m \prod_{k=1}^{l_j} (p^{(i,j)}_k)^{z^{(j)}_{k}}$. The same applies to $P(\mathbf{z}|A_i=0)$. We use this in Eq.~\eqref{eq:posterior} to obtain
\begin{align}
    P(A_i=1|\mathbf{z}) = \sigma \left(\sum_{j=1}^m \sum_{k=1}^{l_j} z^{(j)}_{k} \ln \frac{p^{(i,j)}_k}{q^{(i,j)}_k}+ \ln \frac{P(A_i=1)}{P(A_i=0)}\right).
    \label{eq:posterior2}
    \end{align}
We can merge the indices $j$ and $k$ into a single index $j$, which runs over all $l$ categories
    \begin{align}
    P(A_i=1|\mathbf{z})= \sigma \left(\sum_{j=1}^{l} z_{j} \ln \frac{p^{(i)}_j}{q^{(i)}_j}  + \ln \frac{P(A_i=1)}{P(A_i=0)}\right).
   \end{align}
We define $w_{ij}= \ln \frac{p^{(i)}_j}{q^{(i)}_j}$ and $b_i=\ln \frac{P(A_i=1)}{P(A_i=0)}$ and get
\begin{align}
   P(A_i=1|\mathbf{z})= \sigma \left(\sum_{j=1}^{l} w_{ij} z_{j}  + b_i \right) = \sigma \left(\mathbf{w}_i^{\top} \mathbf{z} + b_i\right).
\label{eq:posteriorOneTask_app}
\end{align}
It is convenient to stack the posteriors of all tasks together. We thus define $P(A|\mathbf{z}) \in [0,1]^n$ as a vector with the $i$-th component given by $P(A_i=1|\mathbf{z})$.
Therefore, we can write
\begin{equation}
    \label{eq:posteriorAllTask_with_b_app}
    P(A|\mathbf{z}) = \sigma \left(W \mathbf{z} + \mathbf{b}\right),
\end{equation}  
with $W=[w_{ij}]_{\substack{i=1,\dots,n\\j=1,\dots ,l}}$ and $\mathbf{b}=(b_1, \dots, b_n)^\top$.~ 
Note however that, since each $\mathbf{z}\in \mathcal{Z}$ has exactly $m$  1-entries, $\mathbf{b}$ can be incorporated into $W$. To incorporate $\mathbf{b}$ into $W$, one can add $\frac{b_i}{m}$ to each each entry in the $i$-th row of $W$ for each $i$.
Hence, we can write
\begin{equation}
    \label{eq:posteriorAllTask_app}
    P(A|\mathbf{z}) = \sigma \left(W \mathbf{z}\right),
\end{equation}
for this re-definition of W.

\paragraph{Definition of structured permutation:} 
We define the latent structure with the help of the vector $\bm{\ell}$ of the number of latent categories for each latent factor
\begin{equation}
  \bm{\ell}=(l_1,\dots ,l_m).
\end{equation}
We furthermore define the function $t_{\bm{\ell}}:\{0,\dots,m\}\rightarrow \mathbb{N}$ as
\begin{equation}
  t_{\bm{\ell}}(k)=\sum_{i=1}^k l_i.
\end{equation}
Note that $t_{\bm{\ell}}(k-1)+1,\dots,t_{\bm{\ell}}(k)$ are the indices of latent variable $Z_k$ in $\mathbf{z}$. 
We now formally define a structured permutation on the structure $\bm{\ell}$ in terms of valid permutation matrices $R$ on that structure such that $\hat{\mathbf{z}}^\top=\mathbf{z}^\top R$ is a structured permutation of $\mathbf{z}$. Let $\pi_d$ denote the set of permutations of the vector $(1,\dots, d)^\top$. We define the set of valid permutations $\pi_{\bm{\ell}}$ of latents in structure $\bm{\ell}$ as 
\begin{equation}
    \pi_{\bm{\ell}}=\{\mathbf{s} | \mathbf{s}\in \pi_m \text{ and } l_i=l_{s_i} \forall i\in\{1,\dots,m\}\}. 
\end{equation}

$R$ is a structured permutation matrix for the structure $\bm{\ell}$ if
\begin{itemize}
    \item $R$ is a permutation matrix, and
    \item there exists an $\mathbf{s}\in \pi_{\bm{\ell}}$ such that the submatrix $[r_{ij}]_{\substack{i=t_{\bm{\ell}}(s_i-1)+1,\dots,t_{\bm{\ell}}(s_i) \\ j=t_{\bm{\ell}}(i-1)+1,\dots, t_{\bm{\ell}}(i)}}$ of $R$ is a permutation matrix for all $i\in\{1,\dots, m\}$.
\end{itemize}

\subsection{Proof of Theorem \ref{thm:Theorem_main}}
\label{app:Proof_permutation}

In this section, we prove Theorem \ref{thm:Theorem_main} from Section \ref{sec:theoresults}. In the main text, we wrote $\hat C$ for the code matrix derived from the multi-WTA bottleneck. To avoid cluttered notation, we write $D$ for this matrix here. Recall from the main text: We say that a matrix $C$ has an $m$-fold one-hot column-block structure if every row of $C$ consists of $m$ consecutive blocks where each block has exactly one $1$ entry and the other entries in the block are $0$.
We first prove the following intermediate lemma, where $\bm{1}_p$ denotes the all-ones column vector of size $p$.
\begin{lemma}
\label{lem:lem1}
    Let $m, l_\textnormal{c}$ be natural numbers with $l_\textnormal{c}\ge 2$. Let $l=m l_\textnormal{c}$ and $p=l_\textnormal{c}^m$. Let $C,D \in \{0,1\}^{p \times l}$ be two matrices with an $m$-fold one-hot column-block structure. Both $C$ and $D$ consist of all possible rows of that type. 
Assume that $C= D G$ for some matrix $G\in \mathbb{R}^{l \times l}$. Then, $C$ can be expressed as $C = D Q + \bm{1}_p \mathbf{b}^\top$ for a matrix $Q\in\{0,1\}^{l \times l}$ and a vector $\mathbf b\in\mathbb{Z}^l$. Furthermore, tile each column of $Q$ into blocks of size $l_\textnormal{c}$. Then, in every column of $Q$, no block consists entirely of $1$'s.
\end{lemma}

\begin{proof}
    
Denote by $\mathbf{d}^{\text{row},k}$ ($\mathbf{d}^{\text{col},k}$) the $k$-th row (column) of a matrix $D$. We have 
\begin{equation}
\label{eq:inner_prod}
    c_i^{\text{row},k}= \mathbf{d}^{\text{row},k} \mathbf{g}^{\text{col},i}.
\end{equation}
Consistent with the $m$-fold one-hot column block structure of $D$, we can divide $D$ into $m$ column blocks and index these blocks by $1,\dots,m$.
We see from Eq.~\eqref{eq:inner_prod} that each column block corresponds to a block of rows in $G$ in the sense that the vector of the $\beta$-th block in a row of $D$ is multiplied with the vector of the $\beta$-th block in a column of $G$. We thus divide the rows of $G$ in such blocks and index them the same way by $1,\dots,m$.

Consider an arbitrary column block $\beta$ of $D$ and an arbitrary column $i$ of $G$. Since, by definition, $D$ contains all one-hot-block combinations, there exist $l_c$ rows in $D$ which are identical outside the column block $\beta$ and necessarily different within the column block. In the following, we consider only these $l_c$ rows of $D$ and call the set of these row indices $K$. Denote by $\alpha_k^{\beta} \in \{1,...,l_c\}$ the index within column block $\beta$ such that $d^{\text{row},k,\beta}_{\alpha_k^{\beta}}=1$, where $\mathbf{d}^{\text{row},k,\beta}$ is a vector containing block $\beta$ from row $k$ of $D$. For all $\alpha_k^{\beta} \in \{1,...,l_c\}$ it holds that
\begin{equation}
    \label{eq:twovalues01}
    c_i^{\text{row},k} = const^{i,\beta} + g^{\text{col},i,\beta}_{\alpha_k^{\beta}} \in \{0,1\},
\end{equation}
where $k \in K$ and $const^{i,\beta}$ is a constant term for a specific $i$ and $\beta$. $\mathbf{g}^{\text{col},i,\beta}$ is a vector containing block $\beta$ from column $i$ of $G$. Eq.~\eqref{eq:twovalues01} implies that $g^{\text{col},i,\beta}_{\alpha_k^{\beta}} \in \{-const^{i,\beta},1-const^{i,\beta}\}$. Therefore, $g^{\text{col},i,\beta}_{\alpha_k^{\beta}}$ can take only two different values, a base value $g_\text{base}^{\text{col},i,\beta}=-const^{i,\beta}$ or the value $g_\text{base}^{\text{col},i,\beta}+1$.

Hence, we have
\begin{equation}
     \label{eq:binarizeQ_vec}
    c_i^{\text{row},k}= \mathbf{d}^{\text{row},k} \mathbf{g}^{\text{col},i} = \sum_{\beta=1}^m g_\text{base}^{\text{col},i,\beta} + \mathbf{d}^{\text{row},k} \mathbf{q}^{\text{col},i} = g_\text{base}^{\text{col},i} + \mathbf{d}^{\text{row},k} \mathbf{q}^{\text{col},i},
\end{equation}
with $\mathbf{q}^{\text{col},i} \in \{0,1\}^l$ and $g_\text{base}^{\text{col},i}=\sum_{\beta=1}^m g_\text{base}^{\text{col},i,\beta}$. We compile $g_\text{base}^{\text{col},i}$ in a vector $\mathbf{b}$ such that $b_i=\sum_{\beta=1}^m g_\text{base}^{\text{col},i,\beta}$ and write this in matrix form:
\begin{equation}
    \label{eq:binarizeQ}
    C = D Q + \bm{1}_p \mathbf{b^\top}.
\end{equation}
Now assume that $\mathbf{q}^{\text{col},i}$ in Eq.~\eqref{eq:binarizeQ_vec} has a block that consists of only $1$'s. This adds a contribution of $1$ independent of $\mathbf{d}$. Therefore, we replace all blocks of column $i$ consisting only of $1$'s of $\mathbf{q}^{\text{col},i}$ with blocks consisting only of $0$'s and add for each of these blocks a $1$ to $b_i$. We do this for all columns $i$. Since $C, D, Q$ are integer matrices, it follows that $\mathbf{b} \in \mathbb{Z}^l$.
\end{proof}

Using this result, we prove Theorem \ref{thm:Theorem_main} from the main text. We reiterate Theorem 1 here (with $D$ replacing $\hat C$ from the main text).
\setcounter{theorem}{0}
\begin{theorem}
\label{thm:Theorem_app}
    Let $m, l_\textnormal{c}$ be natural numbers with $l_\textnormal{c}\ge 2$. Let $l=m l_\textnormal{c}$ and $p=l_\textnormal{c}^m$. Let $C, D \in \{0,1\}^{p \times l}$ be two matrices with an $m$-fold one-hot column-block structure. Both $C$ and $D$ consist of all possible rows of that type. 
Assume that $C= D G$ for some matrix $G\in \mathbb{R}^{l \times l}$. Then, the columns of $D$ are a structured permutation of the columns of $C$.
\end{theorem}

\begin{proof}
We know from Lemma \ref{lem:lem1} that we can write $C$ as $C = D Q + \bm{1}_p \mathbf{b}^\top$ for a matrix $Q\in\{0,1\}^{l \times l}$ and a vector $\mathbf{b}\in\mathbb{Z}^l$. Furthermore, we know that there is no block in the column vectors of $Q$ that consists of only $1$'s.
We first show that we can write $C = D Q$, i.e., we show that $\mathbf{b}=\bm{0}$.

First note that the column sums of $C$ and $D$ are identical. Each column has exactly $s = l_\textnormal{c}^{m-1}$ ones (since in each block, each of the $l_\textnormal{c}$ positions appears in exactly $l_\textnormal{c}^{m-1}$ rows).

This gives a formula for $\mathbf{b}$.
Summing over rows,
\begin{equation}
    \bm{1}_p^\top C = \bm{1}_p^\top D Q + p \mathbf{b}^\top.
\end{equation}

Since $\bm{1}_p^\top D = \bm{1}_p^\top C = s \bm{1}_l^\top$, we get
\begin{equation}
    p \mathbf{b}^\top = s(\bm{1}_l^\top - \bm{1}_l^\top Q).
\end{equation}
Hence, for each column $j$,
\begin{equation}
\label{eq:colsum}  
    b_j = (1 -  \bm{1}_{l}^\top \mathbf{q}^{\text{col},j}) / l_\textnormal{c}  
\end{equation}
(because $p = s l_\textnormal{c}$).
Recall that $b_j \in \mathbb{Z}$.

We determine now the set of values that $(D \mathbf{q}^{\text{col},j})_i$ can take.
Fix a column $j$ of $Q$. Partition $\mathbf{q}^{\text{col},j}$ into $m$ blocks of length $l_\textnormal{c}$. For block $k$, let $a_k$ be the number of ones of $\mathbf{q}^{\text{col},j}$ in that block ($0 \le a_k \le l_\textnormal{c}$).

For any row $\mathbf{d}$ of $D$, the scalar $(\mathbf{d} \mathbf{q}^{\text{col},j})$ equals the number of blocks $k$ for which $\mathbf{d}$ picks a position where $\mathbf{q}^{\text{col},j}$ has a 1. As $\mathbf{d}$ ranges over all rows of $D$ (all choices of one position per block), the attainable values of $\mathbf{d} \mathbf{q}^{\text{col},j}$ form exactly the full integer interval
\begin{equation}
    \{\mathbf{d} \mathbf{q}^{\text{col},j} : \mathbf{d} \text{ is a row of } D\} = \{ s_\text{min}, s_\text{min} + 1, \dots , s_\text{max} \},
\end{equation}
where
\begin{itemize}
    \item $s_\text{min} = |\{k : a_k = l_\textnormal{c}\}|$ (number of blocks that are all-ones),
    \item $s_\text{max} = |\{k : a_k \ge 1\}|$ (number of blocks that have at least one 1).
\end{itemize}
In particular, we have $s_\text{min} =0$ since by construction $\mathbf{q}^{\text{col},j}$ has no block with all-ones (see Lemma \ref{lem:lem1}). Furthermore, $1$ is attainable iff there is at least one block with $a_k \ge 1$. 

Because $C$ has all possible rows of the one-hot-per-block type and $l_\textnormal{c} \ge  2$, each column of $C$ contains both $0$'s and $1$'s across the $p$ rows (its column sum is $s = p/l_\textnormal{c}$). 

Therefore, for each $j$,
$\{ c_{i j} | i=1,\dots,p \} = \{0, 1\}$.
But $C_{i j} = \mathbf{d}^{\text{row},i} \mathbf{q}^{\text{col},j} + b_j$. Hence $\{ \mathbf{d}^{\text{row},i} \mathbf{q}^{\text{col},j} + b_j | i=1,\dots,p \} = \{0, 1\}$ for a constant $b_j$.
From above, we know that the attainable set of $\mathbf{d}^{\text{row},i} \mathbf{q}^{\text{col},j}$ is an interval of consecutive integers $\{s_\text{min}, \dots, s_\text{max}\}$ with $s_\text{min}=0$. Therefore, we have to find a $s_{\text{max}} \in \mathbb{N}$ and a $b_j \in \mathbb{Z}$ such that $\{0, \dots, s_\text{max}\} + b_j = \{0,1\}$ holds. The only way it can be exactly two points that become $\{0,1\}$ after adding $b_j$ is hence $s_\text{max}=1$, $b_j = 0$.

Since this holds for every column $j$, we conclude that $\mathbf{b}=\bm{0}$.

For each column of $C$, we have a column sum of $l_c^{m-1}$ and $l_c^m$ rows. Therefore, each column of $C$ contains both $1$'s and $0$'s. Assume there are two identical columns in one block. Therefore, there exists a row in which both columns contain the value $1$. This contradicts the one-hot column block structure of $C$. Therefore, two different columns within the same block cannot be identical.

Assume two columns from different blocks of $C$ are identical. Since, by definition, $C$ contains all one-hot block combinations, there exists a row in $C$ where only one of the two columns contains the value $1$. This leads to a contradiction, since the two columns are no longer identical. Therefore, two columns from different blocks cannot be identical. Hence, all columns in $C$ are unique, and the same holds for $D$.

We now show that $Q$ is a permutation matrix. From Eq.~\eqref{eq:colsum}, it follows that the column sum of $Q$ equals $1$ for each column. This implies that each column of $Q$ has exactly one $1$. Therefore, for each column $\mathbf{c}^{\text{col},i}$ of $C$ there exists a $j \in \{1,...,l\}$ such that $\mathbf{c}^{\text{col},i} = \mathbf{d}^{\text{col},j}$. Assume there exist two columns $\mathbf{q}^{\text{col},i}$ and $\mathbf{q}^{\text{col},j}$ in $Q$ with $i,j \in \{1,...,l\}$ and $i \neq j$ such that $\mathbf{q}^{\text{col},i}=\mathbf{q}^{\text{col},j}$. It follows that $\mathbf{c}^{\text{col},i} = \mathbf{c}^{\text{col},j}$, which contradicts that all columns in $C$ are unique. Therefore, all columns in $Q$ are unique and contain exactly one $1$. This implies that $Q$ is a permutation matrix.

Finally, we show that $Q$ implements a structured permutation. Consider an arbitrary block $i$ from $C$. Assume that two columns of different blocks of $D$ were mapped to two columns of block $i$ from $C$. Since, by definition, $D$ contains all one-hot block combinations, there exists a row in $D$ where both columns contain a $1$, and this results in a row in $C$ that has two $1$'s in the single block $i$. This contradicts the one-hot column block structure of $C$. Therefore, all columns in a block of $C$ must come from a single block of $D$. Therefore, only the columns within a block can be permuted and since all blocks have the same size, the entire blocks can be permuted.

Hence, the columns of $C$ are a structured permutation of the columns of $D$ (permutation of blocks and within-block permutations), which further implies that also the columns of $D$ are a structured permutation of the columns of $C$.
\end{proof}

\subsection{Details to Section \ref{sec:results_symbolic}.}
\label{app:details_symbolic}
\paragraph{Architecture of $\Phi$}
We used an MLP with $4$ layers and a Leaky ReLU as activation function for $\Phi$. The dimensions of the layers were $(25\rightarrow70\rightarrow70\rightarrow70)$ for the matched setup, $(30\rightarrow100\rightarrow100\rightarrow100)$ for the unmatched setup, and $(60\rightarrow500\rightarrow500\rightarrow500)$ for the experiment with confounding latent variables. The linear layers in the MLPs were randomly initialized from a uniform distribution, which is implemented by default in PyTorch linear layers \citep{pytorchLinearLayer}, and the weights were frozen afterwards.
\paragraph{Architecture of multi-WTA encoder}
We used an MLP with $4$ layers and Leaky ReLU as activation function except for the last layer. Directly before the last layer, we used layer normalization. We used five WTA heads on top of the last layer as described in Section \ref{sec:problem}. Each WTA head consisted of a Gumbel-softmax function with discretized one-hot vectors as output. The one-hot vectors of all WTA heads were stacked together to obtain $\hat{\mathbf{z}}$. The layer dimensions of the MLP were $(70\rightarrow 70\rightarrow70\rightarrow25)$ for the matched setup, $(100\rightarrow100\rightarrow100\rightarrow50)$ for the unmatched setup, and $(500\rightarrow500\rightarrow500\rightarrow50)$ for the experiment with confounding latent variables.
\paragraph{Architecture of readout layer}
We used a linear layer without a bias to project $\hat{\mathbf{z}}$ to the $n$ tasks. Finally, a sigmoid activation function is applied on top to predict the $n$ tasks.
\paragraph{Task generation}
We generated a task according to the following procedure. We sampled for the $i$-th task a categorical distribution parametrization $\mathbf{p}^{(i,j)}$ for each latent variable $j$ using a Dirichlet distributions with $\alpha=1$ for all categories. We stacked all $\mathbf{p}^{(i,j)}$ together to obtain $\mathbf{p}^{(i)}$. We repeated this process a second time to obtain $\mathbf{q}^{(i)}$. We sampled the prior probability $P(A_i=1)$ from a uniform distribution on the interval $(0,1)$. Subsequently we computed $w_{ij}= \ln \frac{p^{(i)}_j}{q^{(i)}_j}$ and $b_i=\ln \frac{P(A_i=1)}{P(A_i=0)}$ and calculated the posterior for each sample according to Eq.~\eqref{eq:posteriorOneTask_app} to get the target value for each sample.

Input vectors were produced based on vectors $\mathbf{z}$ that were sampled from categorical distributions with uniform probabilities over categories.

\paragraph{Hyperparameters for multi-task learning} We used the hyperparameters shown in Table \ref{tab:multi_task_hyperparameter} to train the models in the multi-task learning setup. To enable the models to find the optimal solution, we used high learning rates and a large temperature $\tau$ within the Gumbel-softmax. The high learning rate and the large temperature enable the model to explore different mappings between $\mathbf{z}$ and $\hat{\mathbf{z}}$. We also used an exponential decay for $\tau$ and cosine annealing for the learning rate to reduce exploration towards the end of training and converge to a final model.

For multi-task learning, networks were trained on $100,000$ input samples for $1,000$ epochs and tested the network on $10,000$ samples. We performed this training procedure five times with different random seeds which lead to different network initializations and different shuffles of the training samples.

\begin{table}[t]
    \caption{Hyperparameter used for multi-task training. We used AdamW as optimizer \citep{loshchilov2018decoupled}.}
    \label{tab:multi_task_hyperparameter}
    \centering
    \begin{tabular}{lrrr}
    \toprule
    & \multicolumn{3}{c}{Setup}\\
    \cmidrule(lr){2-4}
    Hyperparameter & unmatched & unmatched confounding & matched\\
    \midrule
        optimizer & AdamW  & AdamW & AdamW\\
        optimizer $\beta_1,\beta_2$ & 0.9, 0.999 & 0.9, 0.999 & 0.9, 0.999\\
        learning rate & 0.020 & 0.005 & 0.008\\
        learning rate scheduler & cos. annealing & cos. annealing & cos. annealing\\
        learning rate scheduler $\eta_{\mathrm{min}}$ & $10^{-6}$ & $10^{-6}$ & $10^{-6}$\\
        learning rate scheduler $T_{\mathrm{max}}$ & 1000 & 2000 & 1000\\
        weight decay & 0.0 & 0.0 & 0.0\\
        number of epochs & 1000 & 2000 & 1000\\
        Gumbel softmax $\tau$ & 20 & 20 & 6\\
        Gumbel softmax $\tau$ exp. decay & 0.999 & 0.999 & 0.999\\
        number of tasks & 50 & 50 & 25\\
    \bottomrule
    \end{tabular}
\end{table}

\paragraph{Evaluation of symbolic representation}
We evaluated whether a representation is symbolic by checking the conditions given in {\em Definition of symbolic representations} in Appendix \ref{app:Problem}.

In addition, we report the number of factors for which all its categories are encoded by a single WTA head as \textit{localized factors}.

\paragraph{Task performance metrics}
Task performance was evaluated using the AUC. The target values for a task are probability values ranging from $0$ to $1$. For the AUC, we dichotomized the probability value. Values greater than or equal to $0.5$ were set to $1$, and values less than $0.5$ were set to $0$. The dichotomized probability values and the model output were used to calculate the AUC. 

To evaluate how closely the target probabilities align with the model outputs during multi-task training, we used the mean absolute error (MAE).
The tasks were considered perfectly solved if the MAE on the test set was less than $10^{-6}$. This threshold was raised in Section \ref{sec:results_confound} to $10^{-3}$ due to the additional noise in $\mathbf{x}$.

\subsection{Details to Section \ref{sec:results_OOD}.}
\label{app:detials_OOD}

\paragraph{Architecture of readout MLP}
We used for $\mathrm{MLP}_{\mathbf{x}}$ and $\mathrm{MLP}_{\hat{\mathbf{z}}}$ an MLP with Leaky ReLU activation function after each layer except for the last layer which consisted of a single neuron with a logistic sigmoid nonlinearity. The number and dimensions of hidden layers were not fixed but determined by hyperparameter optimization, see paragraph {\em Training procedure} below. This ensured a fair comparison as the hyperparameter optimization chooses the hyperparameter with best validation binary cross entropy. 

\paragraph{Train-test splits for generalization evaluations}
We split the input samples $\mathcal{Z}$ into training set $\mathcal{Z}_{\text{train}}$, validation set $\mathcal{Z}_{\text{val}}$, and test set $\mathcal{Z}_{\text{test}}$ in three different ways, in order to
test different types of generalization in our models. 
For each split type, we had  $\mathcal{Z}_{\text{train}} \cup \mathcal{Z}_{\text{val}} \cup \mathcal{Z}_{\text{test}}=\mathcal{Z}$ and $\mathcal{Z}_{\text{train}} \cap \mathcal{Z}_{\text{val}} \cap \mathcal{Z}_{\text{test}}=\{\}$.

For the first split type (random train-test split),
we randomly chose $500$ samples for the validation and $500$ for the test set. The other $4400$ samples were used for the training set. Hence, the test set consisted of category combinations of the latent variables which were not part of the training process.

In the second split type (pair-of-categories train-test split), we made sure that the models were tested on a test set which consisted of samples with a particular category combination of the first two latents which has never been seen during training.
To this end, we selected the first category of the first latent variable and the first category of the second latent variable and used all $135$ samples, which contain this combination, for the test set.  We used the second category of the first latent variable and the second category of the second latent variable and used all $135$ samples which contain this combination for the validation set. The other $5130$ samples which were not part of test and validation set were used for the training set.

For the third split type (constant-category train-test split),
we selected the first category of the first latent variable and used all $1080$ samples which contain this category for the training set. We selected the second category of the first latent variable and used all $1080$ samples which contain this category for the validation set. The other $3240$ samples which were not part of training and validation set were used for the test set.

\paragraph{Training procedure}
We used a Bayesian hyperparameter optimization approach (Tree-structured Parzen Estimator) \citep{optunaHPO} to find appropriate hyperparameter for $\mathrm{MLP}_{\mathbf{x}}$ and $\mathrm{MLP}_{\hat{\mathbf{z}}}$. For each experiment we used $20$ trials to find an appropriate number of layers,  dimension size of hidden layers, learning rate, and so on. The optimized hyperparameters and the search space is depicted in Table \ref{tab:hpo_search_space}.
If the hyperparameter dropout value was greater than $0$, we used a dropout layer after the Leaky ReLU activation in each hidden layer.
If the hyperparameter layernorm was True, we used a layer normalization on the input.
In each trial, an MLP was trained for a maximum of $500$ epochs, where training terminated early if the validation loss did not improve within $10$ epochs. The hyperparameter from the trial that performed best on the validation set were used for the final run. The resulting MLP was then evaluated on the test set. 

We performed this training procedure five times with different random seeds which lead to different network initializations and differently shuffles of the training samples. In addition, for each seed, an independently sampled task was generated on which the generalization was evaluated.

\begin{table}[t]
    \caption{Hyperparameters and the corresponding search space used for hyperparameter optimization.}
    \label{tab:hpo_search_space}
    \centering
    \begin{tabular}{lrrr}
    \toprule
    Hyperparameter & Range\\
    \midrule
        hidden dim. & $\{16,32,64,128,256\}$\\
        number of hidden layers & $\{1,2,3,4,5\}$\\
        dropout & $\{0,0.1,0.2\}$\\
        layernorm & $\{\mathrm{True}, \mathrm{False}\}$\\
        learning rate & $[10^{-4},10^{-1}]$\\
        weight decay & $[10^{-6},10^{-2}]$\\
        batch size & $\{16,32,64,128,256\}$\\
    \bottomrule
    \end{tabular}
\end{table}

\subsection{Details to Section \ref{sec:results_dsprites}}
\label{app:details_dsprites}

\paragraph{Dsprites data set} The dsprites data set introduced in \cite{dsprites17} is built on 5 latent factors: shape, scale, orientation, x-, and y-position. Scale and orientation are not suitable as categorical features, so they were set to a fixed value (largest scale and no rotation). 
The position factors, with 32 values in each direction, were subsampled evenly to 8 possible position in each dimension. To increase the variety of possible feature combinations, we introduced color to the dataset, increasing the dimensions from $64 \times 64$ to $64 \times 64 \times 3$ and assigning a set of visually distinct colors, generated by heuristically selecting points from a coarse RGB lattice, emphasizing variation across color channels. This results in a total of $1920$ samples. 

To test generalization, the pair-of-categories train-test split described above in Section \ref{app:detials_OOD} was used as follows: 
\begin{itemize}
    \item All samples with elliptic shape and the first 3 colors were used as validation set.
    \item All samples with rectangular shape and the subsequent 3 colors were used as test set.
    \item All remaining samples were used as training set. 
\end{itemize}
Hence, the validation and test set each consisted of $3 \cdot 8 \cdot 8 = 192$ samples, 10\% of the total data set.

\paragraph{Architecture} The network architecture was similar to Section~\ref{sec:results_symbolic} and \ref{sec:results_OOD}, described in \ref{app:details_symbolic} and \ref{app:detials_OOD}. It only differed slightly in the encoder architecture, which used 4 layers $(128 \rightarrow 128 \rightarrow 128 \rightarrow 40)$, followed by 4 WTA heads of size 10.

\paragraph{Training procedure}
We used the same training procedure as in Section~\ref{sec:results_symbolic} and \ref{sec:results_OOD}, described in \ref{app:details_symbolic} and \ref{app:detials_OOD}.

\subsection{Supplementary results for Section \ref{sec:results_symbolic}}
\label{app:results_symbolic}
Figure~\ref{fig:5x5mapping} illustrates the structured permutation learned by one of the models in the matched setup, for which a fully symbolic representation emerges. 

Figure~\ref{fig:heatmap_fail_unmatched} shows the empirical conditional activation heatmap for the network in the unmatched setup where $28$ out of the $30$ latent categories were symbolically represented.

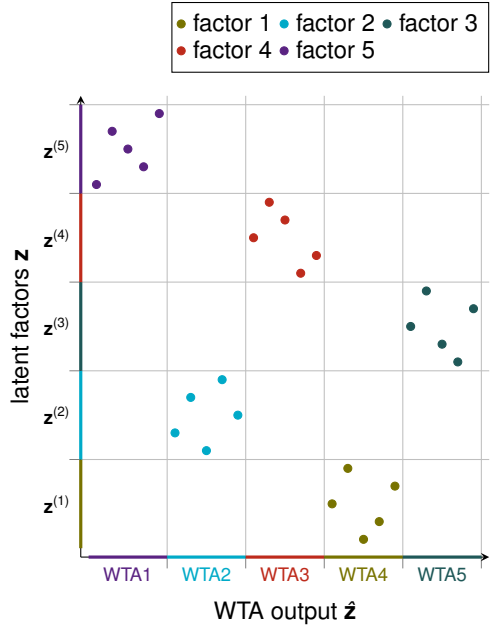
\begin{figure}[t]
    \centering
    \definecolor{wtaone}{RGB}{90,35,130}
    \definecolor{wtatwo}{RGB}{0,170,200}
    \definecolor{wtathree}{RGB}{190,43,28}
    \definecolor{wtafour}{RGB}{120,115,0}
    \definecolor{wtafive}{RGB}{31,88,88}
    
    \begin{tikzpicture}
        \begin{axis}[
            width=0.5\linewidth,
            height=0.55\linewidth,
            xmin=0, xmax=26,
            ymin=0, ymax=26,
            label style={font=\small \sffamily \sansmath},
            tick label style={font=\small \sffamily \sansmath},
            ylabel={latent factors $\mathbf{z}$},
            ylabel style={yshift=5pt},
            xlabel={WTA output $\hat{\mathbf{z}}$},
            xtick={5.5,10.5,15.5,20.5,25.5},
            ytick={5.5,10.5,15.5,20.5,25.5},
            xticklabel style={
            yshift=-2pt,
            },
            xticklabel=\empty,
            yticklabel=\empty,
            grid=major,
            axis lines=left,
            tick align=outside,
            clip=false,
            legend columns=3,
            legend style={
            at={(1,1.05)},
            inner sep=2pt,
            anchor=south east,
            font=\small \sffamily \sansmath,
            row sep=-2pt,
            /tikz/column sep=1pt
            },
        ]

        \addplot[
            only marks,
            mark=*,
            mark size=1.5pt,
            wtafour
        ] table[
        x=wta,
        y=z,
        col sep=comma
        ] {WTA4_5x5_mapping_baseline_30tasks_results.csv};
        \addlegendentry{factor $1$};

        \addplot[
            only marks,
            mark=*,
            mark size=1.5pt,
            wtatwo
        ] table[
        x=wta,
        y=z,
        col sep=comma
        ] {WTA2_5x5_mapping_baseline_30tasks_results.csv};
        \addlegendentry{factor $2$};

        \addplot[
            only marks,
            mark=*,
            mark size=1.5pt,
            wtafive
        ] table[
        x=wta,
        y=z,
        col sep=comma
        ] {WTA5_5x5_mapping_baseline_30tasks_results.csv};
        \addlegendentry{factor $3$};

        \addplot[
            only marks,
            mark=*,
            mark size=1.5pt,
            wtathree
        ] table[
        x=wta,
        y=z,
        col sep=comma
        ] {WTA3_5x5_mapping_baseline_30tasks_results.csv};
        \addlegendentry{factor $4$};
        
        \addplot[
            only marks,
            mark=*,
            mark size=1.5pt,
            wtaone
        ] table[
        x=wta,
        y=z,
        col sep=comma
        ] {WTA1_5x5_mapping_baseline_30tasks_results.csv};
        \addlegendentry{factor $5$};

        \draw[wtaone, very thick]  (axis cs:0.5,0) -- (axis cs:5.5,0);
        \draw[wtatwo, very thick]  (axis cs:5.5,0) -- (axis cs:10.5,0);
        \draw[wtathree, very thick](axis cs:10.5,0) -- (axis cs:15.5,0);
        \draw[wtafour, very thick] (axis cs:15.5,0) -- (axis cs:20.5,0);
        \draw[wtafive, very thick] (axis cs:20.5,0) -- (axis cs:25.5,0);
    
        \node[wtaone,font=\scriptsize]   at (axis cs:3,-1) {\sffamily{WTA1}};
        \node[wtatwo,font=\scriptsize]   at (axis cs:8,-1) {\sffamily{WTA2}};
        \node[wtathree,font=\scriptsize] at (axis cs:13,-1) {\sffamily{WTA3}};
        \node[wtafour,font=\scriptsize]  at (axis cs:18,-1) {\sffamily{WTA4}};
        \node[wtafive,font=\scriptsize]  at (axis cs:23,-1) {\sffamily{WTA5}};

        \node[font=\scriptsize \sansmath]   at (axis cs:-1.5,3) {$\mathbf{z}^{(1)}$};
        \node[font=\scriptsize \sansmath]   at (axis cs:-1.5,8) {$\mathbf{z}^{(2)}$};
        \node[font=\scriptsize \sansmath] at (axis cs:-1.5,13) {$\mathbf{z}^{(3)}$};
        \node[font=\scriptsize \sansmath]  at (axis cs:-1.5,18) {$\mathbf{z}^{(4)}$};
        \node[font=\scriptsize \sansmath]  at (axis cs:-1.5,23) {$\mathbf{z}^{(5)}$};
    
        \draw[wtafour, very thick]  (axis cs:0,0.5) -- (axis cs:0,5.5);
        \draw[wtatwo, very thick] (axis cs:0,5.5) -- (axis cs:0,10.5);
        \draw[wtafive, very thick] (axis cs:0,10.5) -- (axis cs:0,15.5);
        \draw[wtathree, very thick](axis cs:0,15.5) -- (axis cs:0,20.5);
        \draw[wtaone, very thick]  (axis cs:0,20.5) -- (axis cs:0,25.5);
    
        \end{axis}
    \end{tikzpicture}
    \caption{Mapping from $\hat{\mathbf{z}}$ to $\mathbf{z}$ after multi-task learning (first seed, matched setup) on $n=30$ tasks. The model solved all tasks perfectly after training. Each factor consists of $5$ categories and each WTA head consists of $5$ outputs. $\hat{\mathbf{z}}$ is a structured permuted version of $\mathbf{z}$. Therefore, a symbolic representation emerged in the model.}
    \label{fig:5x5mapping}
\end{figure}

\begin{figure}[t]
    \centering
    \begin{tikzpicture}
                \begin{axis}[
                width=0.8\linewidth,
                height=0.5\linewidth,
                colormap/viridis,
                colorbar,
                axis on top,
                xlabel={},
                ylabel={},
                ylabel style={yshift=-5pt},
                label style={font=\small \sffamily \sansmath},
                tick label style={font=\small \sffamily \sansmath},
                enlargelimits=false,
                xticklabel=\empty,
                yticklabel=\empty,
                xtick={9.5,19.5,29.5,39.5},
                ytick={4.5,12.5,17.5,20.5},
                grid=major,
                clip=false
                ]
                \addplot[
                matrix plot*,
                point meta=explicit,
                mesh/cols=50
                ]
                table[
                x=zhat,
                y=z,
                meta=value,
                col sep=comma,
                ]{symbolic_rep_5_1ts_fail.csv};

                \node[font=\scriptsize \sansmath]   at (axis cs:-3,2) {$\mathbf{z}^{(1)}$};
                \node[font=\scriptsize \sansmath]   at (axis cs:-3,8.5) {$\mathbf{z}^{(2)}$};
                \node[font=\scriptsize \sansmath] at (axis cs:-3,15) {$\mathbf{z}^{(3)}$};
                \node[font=\scriptsize \sansmath]  at (axis cs:-3,19) {$\mathbf{z}^{(4)}$};
                \node[font=\scriptsize \sansmath]  at (axis cs:-3,25) {$\mathbf{z}^{(5)}$};

                \node[font=\scriptsize \sansmath]   at (axis cs:5,-2) {$\hat{\mathbf{z}}^{(1)}$};
                \node[font=\scriptsize \sansmath]   at (axis cs:15,-2) {$\hat{\mathbf{z}}^{(2)}$};
                \node[font=\scriptsize \sansmath] at (axis cs:25,-2) {$\hat{\mathbf{z}}^{(3)}$};
                \node[font=\scriptsize \sansmath]  at (axis cs:35,-2) {$\hat{\mathbf{z}}^{(4)}$};
                \node[font=\scriptsize \sansmath]  at (axis cs:45,-2) {$\hat{\mathbf{z}}^{(5)}$};
                
                \end{axis}
            \end{tikzpicture}
    \caption{The heatmap shows the empirical conditional activation between input latent categories $\mathbf{z}^{(i)}$ and output latent dimensions $\hat{\mathbf{z}}^{(j)}$, where each entry is given by $P(\mathbf{\hat{z}}^{(j)} = 1 \mid \mathbf{z}^{(i)} = 1)$. Here, $\hat{\mathbf{z}}$ is from a trained model in the unmatched setup where a perfect symbolic representation did not emerge in $\hat{\mathbf{z}}$. In this specific training run, two neurons within a single WTA head $\hat{\mathbf{z}}^{(3)}$ (numbers 5 and 10) jointly represented two categories of the second latent factor $\mathbf{z}^{(2)}$. If one of these two neurons was active, the value of $\mathbf{z}^{(2)}$ could not be decoded unambiguously. This ambiguity could be resolved with the help of WTA head $\hat{\mathbf{z}}^{(2)}$, where selective conditional activations can be seen for the two ambiguous categories of $\mathbf{z}^{(2)}$.}
    \label{fig:heatmap_fail_unmatched}
\end{figure}
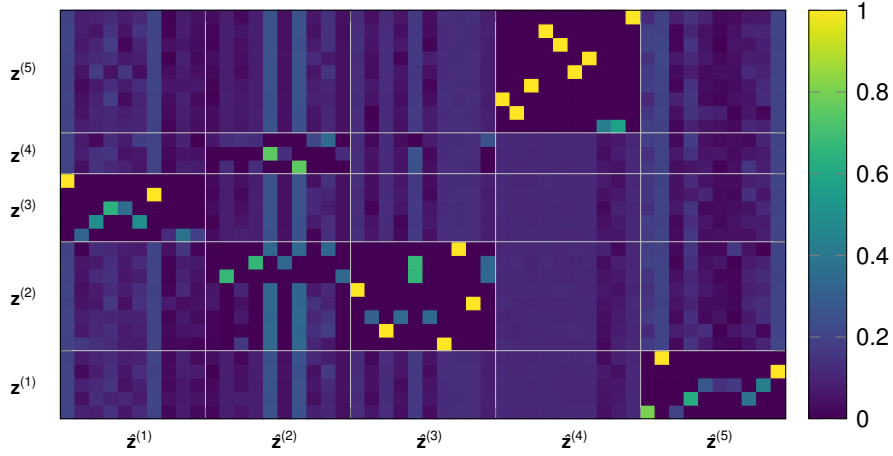

\subsection{Supplementary results for Section \ref{sec:results_OOD}}
\label{app:results_OOD}
Table~\ref{tab:cauc_random_combinations} shows the detailed AUC metrics for training $\mathrm{MLP}_{\mathbf{x}}$ and  $\mathrm{MLP}_{\mathbf{\hat{z}}}$ on random train-test splits. The same metrics are shown in Table~\ref{tab:auc_2combi} for the pair-of-categories train-test split. The detailed results for training on the constant-category train-test split are shown in Table~\ref{tab:r2_constant_factor}. 

\begin{table}[H]
    \caption{Train and test performance (AUC) of $\mathrm{MLP}_{\mathbf{x}}$ and $\mathrm{MLP}_{\hat{\mathbf{z}}}$ for the random train-test split with varying number of training samples. The mean and standard deviation of the AUC across $5$ runs with different seeds are reported.
    }
    \label{tab:cauc_random_combinations}
    \centering
    \begin{tabular}{lrrrrrrrr}
    \toprule
    & \multicolumn{4}{c}{Train AUC} & \multicolumn{4}{c}{Test AUC}\\
    \cmidrule(lr){2-5} \cmidrule(lr){6-9}
    \# train samples & Mean $\mathbf{x}$ & SD $\mathbf{x}$ & Mean $\hat{\mathbf{z}}$ & SD $\hat{\mathbf{z}}$ & Mean $\mathbf{x}$ & SD $\mathbf{x}$& Mean $\hat{\mathbf{z}}$ & SD $\hat{\mathbf{z}}$ \\
    \midrule
    100 & 0.896 & 0.092 & 1.000 & <0.001 & 0.796 & 0.065 & 0.993 & 0.005 \\
    500 & 0.952 & 0.009 & 1.000 & <0.001 & 0.900 & 0.027 & 1.000 & <0.001 \\
    1000 & 0.968 & 0.017 & 1.000 & <0.001 & 0.931 & 0.029 & 1.000 & <0.001 \\
    2000 & 0.980 & 0.011 & 1.000 & <0.001 & 0.951 & 0.020 & 1.000 & <0.001 \\
    4400 & 0.991 & 0.007 & 1.000 & <0.001 & 0.973 & 0.010 & 1.000 & <0.001 \\
    \bottomrule
    \end{tabular}
\end{table}

\begin{table}[H]
    \caption{Train and test performance (AUC) of $\mathrm{MLP}_{\mathbf{x}}$ and $\mathrm{MLP}_{\hat{\mathbf{z}}}$ for the pair-of-categories train-test split with varying number of training samples. The mean and standard deviation of the AUC across $5$ runs with different seeds are reported.}
    \label{tab:auc_2combi}
    \centering
    \begin{tabular}{lrrrrrrrr}
    \toprule
    & \multicolumn{4}{c}{Train AUC} & \multicolumn{4}{c}{Test AUC}\\
    \cmidrule(lr){2-5} \cmidrule(lr){6-9}
    \# train samples & Mean $\mathbf{x}$ & SD $\mathbf{x}$ & Mean $\hat{\mathbf{z}}$ & SD $\hat{\mathbf{z}}$ & Mean $\mathbf{x}$ & SD $\mathbf{x}$& Mean $\hat{\mathbf{z}}$ & SD $\hat{\mathbf{z}}$ \\
    \midrule
    100 & 0.771 & 0.118 & 1.000 & <0.001 & 0.737 & 0.119 & 0.987 & 0.006 \\
    500 & 0.756 & 0.145 & 1.000 & <0.001 & 0.718 & 0.131 & 1.000 & <0.001 \\
    1000 & 0.858 & 0.060 & 1.000 & <0.001 & 0.857 & 0.062 & 1.000 & <0.001 \\
    2000 & 0.876 & 0.083 & 1.000 & <0.001 & 0.888 & 0.059 & 1.000 & <0.001 \\
    4400 & 0.917 & 0.074 & 1.000 & <0.001 & 0.917 & 0.059 & 1.000 & <0.001 \\
    \bottomrule
    \end{tabular}
\end{table}

\begin{table}[H]
    \caption{Train and test performance (AUC) of $\mathrm{MLP}_{\mathbf{x}}$ and $\mathrm{MLP}_{\hat{\mathbf{z}}}$ for the constant-category train-test split. The mean and standard deviation of the AUC across $5$ runs with different seeds are reported.}
    \label{tab:r2_constant_factor}
    \centering
    \begin{tabular}{lrrrr}
    \toprule
    & \multicolumn{4}{c}{AUC}\\
    \cmidrule(lr){2-5}
    Set & Mean $\mathbf{x}$ & SD $\mathbf{x}$ & Mean $\hat{\mathbf{z}}$ & SD $\hat{\mathbf{z}}$ \\
    \midrule
    Train & 0.936 & 0.028 & 1.000 & <0.001 \\
    Test & 0.864 & 0.023 & 0.999 & <0.001  \\
    \bottomrule
    \end{tabular}
\end{table}

\subsection{Supplementary results for Section \ref{sec:results_confound}}
\label{app:results_confound}
Figure~\ref{fig:heatmap_fail_unmatched_unused} shows the empirical conditional activation heatmap for the network in the unmatched setup with confounding factors where a symbolic representation emerged, but one latent factor was mapped to two WTA heads.

Table~\ref{tab:confounding_auc_random_combinations} shows the detailed AUC metrics for training $\mathrm{MLP}_{\mathbf{x}}$ and  $\mathrm{MLP}_{\mathbf{\hat{z}}}$ on random train-test splits. The same metrics are shown in Table~\ref{tab:confounding_auc_2combi} for the pair-of-categories train-test split. The detailed results for training on the constant-category train-test split are shown in Table~\ref{tab:confounding_r2_constant_factor}.

\begin{figure}[H]
    \centering
    \begin{tikzpicture}
                \begin{axis}[
                width=0.8\linewidth,
                height=0.5\linewidth,
                colormap/viridis,
                colorbar,
                axis on top,
                xlabel={},
                ylabel={},
                ylabel style={yshift=-5pt},
                label style={font=\small \sffamily \sansmath},
                tick label style={font=\small \sffamily \sansmath},
                enlargelimits=false,
                xticklabel=\empty,
                yticklabel=\empty,
                xtick={9.5,19.5,29.5,39.5},
                ytick={4.5,12.5,17.5,20.5},
                grid=major,
                clip=false
                ]
                \addplot[
                matrix plot*,
                point meta=explicit,
                mesh/cols=50
                ]
                table[
                x=zhat,
                y=z,
                meta=value,
                col sep=comma,
                ]{symbolic_rep_5_1ts_fail_unused.csv};

                \node[font=\scriptsize \sansmath]   at (axis cs:-3,2) {$\mathbf{z}^{(1)}$};
                \node[font=\scriptsize \sansmath]   at (axis cs:-3,8.5) {$\mathbf{z}^{(2)}$};
                \node[font=\scriptsize \sansmath] at (axis cs:-3,15) {$\mathbf{z}^{(3)}$};
                \node[font=\scriptsize \sansmath]  at (axis cs:-3,19) {$\mathbf{z}^{(4)}$};
                \node[font=\scriptsize \sansmath]  at (axis cs:-3,25) {$\mathbf{z}^{(5)}$};

                \node[font=\scriptsize \sansmath]   at (axis cs:5,-2) {$\hat{\mathbf{z}}^{(1)}$};
                \node[font=\scriptsize \sansmath]   at (axis cs:15,-2) {$\hat{\mathbf{z}}^{(2)}$};
                \node[font=\scriptsize \sansmath] at (axis cs:25,-2) {$\hat{\mathbf{z}}^{(3)}$};
                \node[font=\scriptsize \sansmath]  at (axis cs:35,-2) {$\hat{\mathbf{z}}^{(4)}$};
                \node[font=\scriptsize \sansmath]  at (axis cs:45,-2) {$\hat{\mathbf{z}}^{(5)}$};
                
                \end{axis}
            \end{tikzpicture}
    \caption{The heatmap shows the empirical conditional activation between input latent categories $\mathbf{z}^{(i)}$ and output latent dimensions $\hat{\mathbf{z}}^{(j)}$, where each entry represents $P(\mathbf{\hat{z}}^{(j)} = 1 \mid \mathbf{z}^{(i)} = 1)$. $\hat{\mathbf{z}}$ is from a trained model in the unmatched setup with confounding factors that developed a symbolic representation. In this specific training run, one neuron within  WTA head $\hat{\mathbf{z}}^{(2)}$ (neuron 6) represented two categories of the latent factor $\mathbf{z}^{(5)}$. This ambiguity could be resolved with the help of WTA head $\hat{\mathbf{z}}^{(5)}$, where selective conditional activations can be seen for the two ambiguous categories of $\mathbf{z}^{(5)}$. Hence, although the representation was not localized to individual WTA outputs, it was still a symbolic representation.}

    \label{fig:heatmap_fail_unmatched_unused}
\end{figure}
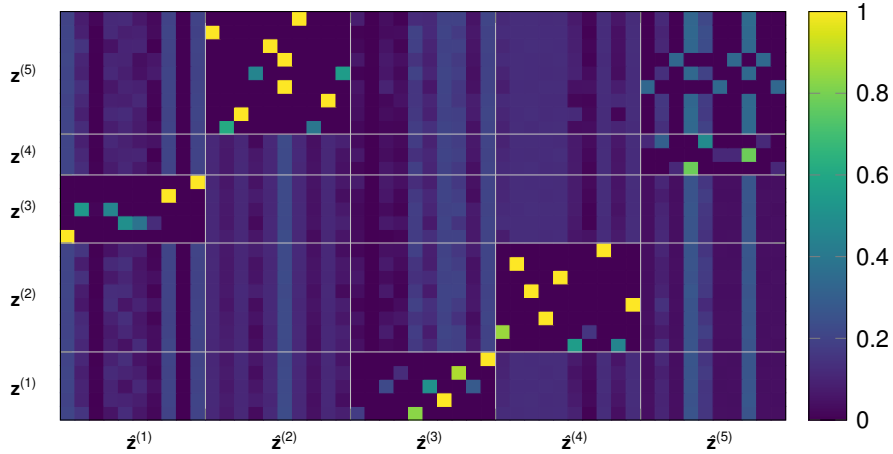

\begin{table}[H]
    \caption{Train and test performance (AUC) of $\mathrm{MLP}_{\mathbf{x}}$ and $\mathrm{MLP}_{\hat{\mathbf{z}}}$ on the random train-test split with varying number of training samples and confounding factors. The mean and standard deviation of the AUC across $5$ runs with different seeds are reported.}
    \label{tab:confounding_auc_random_combinations}
    \centering
    \begin{tabular}{lrrrrrrrr}
    \toprule
    & \multicolumn{4}{c}{Train AUC} & \multicolumn{4}{c}{Test AUC}\\
    \cmidrule(lr){2-5} \cmidrule(lr){6-9}
    \# train samples & Mean $\mathbf{x}$ & SD $\mathbf{x}$ & Mean $\hat{\mathbf{z}}$ & SD $\hat{\mathbf{z}}$ & Mean $\mathbf{x}$ & SD $\mathbf{x}$& Mean $\hat{\mathbf{z}}$ & SD $\hat{\mathbf{z}}$ \\
    \midrule
    100 & 0.822 & 0.205 & 1.000 & <0.001 & 0.699 & 0.151 & 0.995 & 0.004 \\
    500 & 0.976 & 0.004 & 1.000 & <0.001 & 0.903 & 0.012 & 0.999 & 0.001 \\
    1000 & 0.979 & 0.010 & 1.000 & <0.001 & 0.920 & 0.011 & 1.000 & <0.001 \\
    2000 & 0.997 & 0.001 & 1.000 & <0.001 & 0.945 & 0.012 & 1.000 & <0.001 \\
    4400 & 0.996 & 0.003 & 1.000 & <0.001 & 0.970 & 0.005 & 1.000 & <0.001 \\
    \bottomrule
    \end{tabular}
\end{table}

\begin{table}[H]
    \caption{Train and test performance (AUC) of $\mathrm{MLP}_{\mathbf{x}}$ and $\mathrm{MLP}_{\hat{\mathbf{z}}}$ for the pair-of-categories train-test split with varying number of training samples and confounding factors. The mean and standard deviation of the AUC across $5$ runs with different seeds are reported.}
    \label{tab:confounding_auc_2combi}
    \centering
    \begin{tabular}{lrrrrrrrr}
    \toprule
    & \multicolumn{4}{c}{Train AUC} & \multicolumn{4}{c}{Test AUC}\\
    \cmidrule(lr){2-5} \cmidrule(lr){6-9}
    \# train samples & Mean $\mathbf{x}$ & SD $\mathbf{x}$ & Mean $\hat{\mathbf{z}}$ & SD $\hat{\mathbf{z}}$ & Mean $\mathbf{x}$ & SD $\mathbf{x}$& Mean $\hat{\mathbf{z}}$ & SD $\hat{\mathbf{z}}$ \\
    \midrule
    100 & 0.879 & 0.172 & 0.929 & 0.140 & 0.764 & 0.108 & 0.895 & 0.158 \\
    500 & 0.975 & 0.017 & 1.000 & <0.001 & 0.905 & 0.050 & 0.998 & 0.004 \\
    1000 & 0.963 & 0.010 & 1.000 & <0.001 & 0.924 & 0.046 & 0.999 & 0.001 \\
    2000 & 0.984 & 0.010 & 1.000 & <0.001 & 0.942 & 0.026 & 0.999 & 0.001 \\
    4400 & 0.974 & 0.021 & 1.000 & <0.001 & 0.939 & 0.030 & 0.999 & 0.001 \\
    \bottomrule
    \end{tabular}
\end{table}

\begin{table}[H]
    \caption{Train and test performance (AUC) of $\mathrm{MLP}_{\mathbf{x}}$ and $\mathrm{MLP}_{\hat{\mathbf{z}}}$ on the constant-factor train-test split with confounding factors. The mean and standard deviation of the AUC across $5$ runs with different seeds are reported.}
    \label{tab:confounding_r2_constant_factor}
    \centering
    \begin{tabular}{lrrrr}
    \toprule
    & \multicolumn{4}{c}{AUC}\\
    \cmidrule(lr){2-5}
    Set & Mean $\mathbf{x}$ & SD $\mathbf{x}$ & Mean $\hat{\mathbf{z}}$ & SD $\hat{\mathbf{z}}$ \\
    \midrule
    Train & 0.966 & 0.008 & 1.000 & <0.001 \\
    Test & 0.886 & 0.017 & 0.998 & 0.001 \\
    \bottomrule
    \end{tabular}
\end{table}

\subsection{Supplementary results for Section \ref{sec:results_number_of_tasks}}
\label{app:results_num_tasks}
Figure~\ref{fig:5x5mapping20tasks} show the learned structured permutation when using $20$ tasks for training in the multitask-training approach on latents with $25$ categories. Table~\ref{tab:multitasktraining_5x5_tasks} shows the training metrics for  different numbers of tasks.

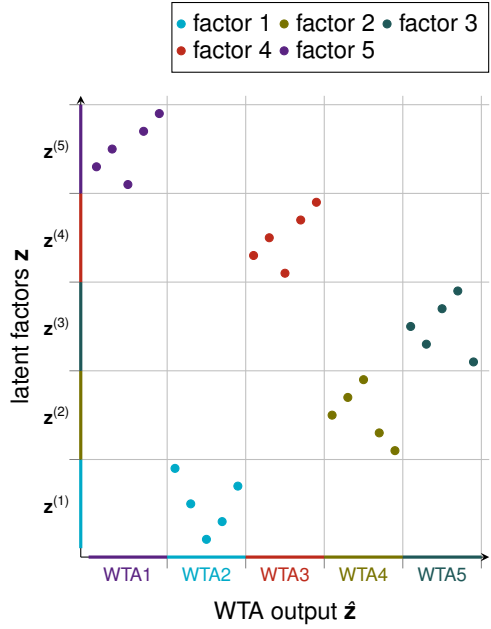
\begin{figure}[t]
    \centering
    \definecolor{wtaone}{RGB}{90,35,130}
    \definecolor{wtatwo}{RGB}{0,170,200}
    \definecolor{wtathree}{RGB}{190,43,28}
    \definecolor{wtafour}{RGB}{120,115,0}
    \definecolor{wtafive}{RGB}{31,88,88}
    
    \begin{tikzpicture}
        \begin{axis}[
            width=0.5\linewidth,
            height=0.55\linewidth,
            xmin=0, xmax=26,
            ymin=0, ymax=26,
            label style={font=\small \sffamily \sansmath},
            tick label style={font=\small \sffamily \sansmath},
            ylabel={latent factors $\mathbf{z}$},
            ylabel style={yshift=5pt},
            xlabel={WTA output $\hat{\mathbf{z}}$},
            xtick={5.5,10.5,15.5,20.5,25.5},
            ytick={5.5,10.5,15.5,20.5,25.5},
            xticklabel style={
            yshift=-2pt,
            },
            xticklabel=\empty,
            yticklabel=\empty,
            grid=major,
            axis lines=left,
            tick align=outside,
            clip=false,
            legend columns=3,
            legend style={
            at={(1,1.05)},
            inner sep=2pt,
            anchor=south east,
            font=\small \sffamily \sansmath,
            row sep=-2pt,
            /tikz/column sep=1pt
            },
        ]

        \addplot[
            only marks,
            mark=*,
            mark size=1.5pt,
            wtatwo
        ] table[
        x=wta,
        y=z,
        col sep=comma
        ] {WTA2_5x5_mapping_baseline_20tasks_results.csv};
        \addlegendentry{factor $1$};

        \addplot[
            only marks,
            mark=*,
            mark size=1.5pt,
            wtafour
        ] table[
        x=wta,
        y=z,
        col sep=comma
        ] {WTA4_5x5_mapping_baseline_20tasks_results.csv};
        \addlegendentry{factor $2$};

        \addplot[
            only marks,
            mark=*,
            mark size=1.5pt,
            wtafive
        ] table[
        x=wta,
        y=z,
        col sep=comma
        ] {WTA5_5x5_mapping_baseline_20tasks_results.csv};
        \addlegendentry{factor $3$};

        \addplot[
            only marks,
            mark=*,
            mark size=1.5pt,
            wtathree
        ] table[
        x=wta,
        y=z,
        col sep=comma
        ] {WTA3_5x5_mapping_baseline_20tasks_results.csv};
        \addlegendentry{factor $4$};

        \addplot[
            only marks,
            mark=*,
            mark size=1.5pt,
            wtaone
        ] table[
        x=wta,
        y=z,
        col sep=comma
        ] {WTA1_5x5_mapping_baseline_20tasks_results.csv};
        \addlegendentry{factor $5$};
        
        \draw[wtaone, very thick]  (axis cs:0.5,0) -- (axis cs:5.5,0);
        \draw[wtatwo, very thick]  (axis cs:5.5,0) -- (axis cs:10.5,0);
        \draw[wtathree, very thick](axis cs:10.5,0) -- (axis cs:15.5,0);
        \draw[wtafour, very thick] (axis cs:15.5,0) -- (axis cs:20.5,0);
        \draw[wtafive, very thick] (axis cs:20.5,0) -- (axis cs:25.5,0);
    
        \node[wtaone,font=\scriptsize]   at (axis cs:3,-1) {\sffamily{WTA1}};
        \node[wtatwo,font=\scriptsize]   at (axis cs:8,-1) {\sffamily{WTA2}};
        \node[wtathree,font=\scriptsize] at (axis cs:13,-1) {\sffamily{WTA3}};
        \node[wtafour,font=\scriptsize]  at (axis cs:18,-1) {\sffamily{WTA4}};
        \node[wtafive,font=\scriptsize]  at (axis cs:23,-1) {\sffamily{WTA5}};

        \node[font=\scriptsize \sansmath]   at (axis cs:-1.5,3) {$\mathbf{z}^{(1)}$};
        \node[font=\scriptsize \sansmath]   at (axis cs:-1.5,8) {$\mathbf{z}^{(2)}$};
        \node[font=\scriptsize \sansmath] at (axis cs:-1.5,13) {$\mathbf{z}^{(3)}$};
        \node[font=\scriptsize \sansmath]  at (axis cs:-1.5,18) {$\mathbf{z}^{(4)}$};
        \node[font=\scriptsize \sansmath]  at (axis cs:-1.5,23) {$\mathbf{z}^{(5)}$};
    
        \draw[wtatwo, very thick]  (axis cs:0,0.5) -- (axis cs:0,5.5);
        \draw[wtafour, very thick] (axis cs:0,5.5) -- (axis cs:0,10.5);
        \draw[wtafive, very thick] (axis cs:0,10.5) -- (axis cs:0,15.5);
        \draw[wtathree, very thick](axis cs:0,15.5) -- (axis cs:0,20.5);
        \draw[wtaone, very thick]  (axis cs:0,20.5) -- (axis cs:0,25.5);
    
        \end{axis}
    \end{tikzpicture}
    \caption{Mapping from $\hat{\mathbf{z}}$ to $\mathbf{z}$ after training the model on $n=20$ tasks (first seed). The model solved all tasks perfectly after training. Each factor consists of $5$ categories and each WTA head consists of $5$ outputs. $\hat{\mathbf{z}}$ is a structured permuted version of $\mathbf{z}$. Therefore, a symbolic representation emerged in the model.}
    \label{fig:5x5mapping20tasks}
\end{figure}

\begin{table}[t]
\caption{The mean, standard deviation (SD), and minimum of the Mean Absolute Error (MAE) after multi-task training on test dataset across $5$ runs with different seeds in the matched setup. The column "\# solved perfectly" shows the number of models that solved the $n$ tasks perfectly after training. The last column depicts how many of the models which solved the $n$ tasks perfectly also have developed a valid symbolic representation. NA stands for not available, because no model solved the $n$ tasks perfectly.}
\label{tab:multitasktraining_5x5_tasks}
\centering
\begin{tabular}{lrrrrr}
\toprule
$n$ tasks & Mean MAE & SD MAE & Min MAE & \makecell{\# solved perfectly\\(max. $5$)}& \makecell{Valid symbolic rep.\\in \%}\\ 
\midrule
1 & 0.0133 & 0.0009 & 0.0123 & 0 & NA\\
5 & 0.0641 & 0.0003 & 0.0636 & 0 & NA\\
10 & 0.0943 & 0.0013 & 0.0927 & 0 & NA\\
15 & 0.0352 & 0.0428 & 0.0002 & 0 & NA\\
20 & 0.0132 & 0.0264 & 8 \texttimes{} 10\textsuperscript{-9} & 3 & 100\\
25 & 0.0141 & 0.0279 & 9 \texttimes{} 10\textsuperscript{-9} & 2 & 100\\
30 & 0.0149 & 0.0296 & 32 \texttimes{} 10\textsuperscript{-9} & 3 & 100\\
\bottomrule
\end{tabular}
\end{table}

\subsection{Supplementary results for Section \ref{sec:results_dsprites}}
\label{app:results_dsprites}

In this section we show the detailed results for experiments with the dsprites data set. Table~\ref{tab:dsprites_pretrain} shows the detailed metrics for training the encoder with WTA heads for each separate run. Table~\ref{tab:auc_dsprites} shows the generalization performance of $\mathrm{MLP}_{\mathbf{x}}$  and  $\mathrm{MLP}_{\mathbf{\hat{z}}}$ for the pair-of-categories train-test split.

\begin{table}[H]
\caption{Training performance on dsprites data set with 5 different seeds.}
\label{tab:dsprites_pretrain}
\centering
\begin{tabular}{lccc}
\toprule
Training run & MAE & \makecell{Symbolic-encoded \\ categories (max. 29)} & \makecell{Localized \\ factors (max. 4)} \\ 
\midrule
1     & 0.01139 & 28 & 2 \\
2     & 0.00018   &  29 & 4 \\
3     & 0.00019    &  29 & 4 \\
4     & 0.00019   &  29 & 4 \\
5     & 0.00020    &  26 & 2 \\
\bottomrule
\end{tabular}
\end{table}

\begin{table}[H]
\caption{Train and test performance ($AUC$) of $\mathrm{MLP}_{\mathbf{x}}$ and $\mathrm{MLP}_{\hat{\mathbf{z}}}$ on the pair-of-categories dsprites dataset with varying number of training samples. The mean and standard deviation of the $AUC$ across $5$ runs with different seeds are reported.}
\label{tab:auc_dsprites}
\centering
\begin{tabular}{lrrrrrrrr}
\toprule
 & \multicolumn{4}{c}{Train AUC} & \multicolumn{4}{c}{Test AUC}\\
\cmidrule(lr){2-5} \cmidrule(lr){6-9}
\# train samples & Mean $\mathbf{x}$ & SD $\mathbf{x}$ & Mean $\hat{\mathbf{z}}$ & SD $\hat{\mathbf{z}}$ & Mean $\mathbf{x}$ & SD $\mathbf{x}$& Mean $\hat{\mathbf{z}}$ & SD $\hat{\mathbf{z}}$ \\
\midrule
100  & 0.978 & 0.035 & 0.992 & 0.018 & 0.562 & 0.114 & 0.794 & 0.124 \\
300  & 0.960 & 0.047 & 1.000 & 0.000 & 0.766 & 0.087 & 0.908 & 0.091 \\
500  & 0.988 & 0.016 & 1.000 & 0.000 & 0.786 & 0.124 & 0.990 & 0.017 \\
1000 & 0.990 & 0.014 & 1.000 & 0.000 & 0.938 & 0.044 & 1.000 & 0.000 \\
1500 & 0.996 & 0.009 & 1.000 & 0.000 & 0.980 & 0.022 & 1.000 & 0.000 \\
\bottomrule
\end{tabular}
\end{table}

\subsection{Computational resources}
\label{app:compute}
We used for our experiments one node with a 192 Core AMD EPYC 9654 CPU and another node with a 128 Core AMD EPYC 7543 CPU and 8 Nvidia A40 GPUs. However, experiments can be run on standard CPUs and a single Nvidia A40. Multiple experiments were run in parallel on these nodes to make efficient use of the available computational resources. A single multi-task training run took between 1 and 5 hours, depending on the hyperparameters. The generalization experiment runs were shorter and typically took between 30 minutes and 2 hours. To reproduce all experiments, we estimate a total amount of 400 GPUh. There were no further compute resources needed for prior studies.

\subsection{Assets}
\label{app:assets}
The main Python packages used in our experiments were PyTorch \citep{pytorch}, PyTorch Lightning \citep{pytorch_lightning}, Optuna \citep{optunaHPO}, and Hydra \citep{Yadan2019Hydra}.

\end{document}